\documentclass[letterpaper, 10 pt, conference]{ieeeconf}  

\IEEEoverridecommandlockouts                              

\usepackage{epsfig} 
\usepackage[cmex10]{amsmath} 
\usepackage{amssymb}  
\usepackage{latexsym}

\usepackage{amsthm}
  
\usepackage[mathcal]{euscript}
\usepackage[vcentermath]{youngtab}

\makeatletter
\let\NAT@parse\undefined
\makeatother

\usepackage{cite}
\usepackage[%
hyperfigures=true,%
breaklinks=true,%
colorlinks=true,%
citecolor=green,%
linkcolor=blue%
]{hyperref}

\usepackage{multirow}

\usepackage{color}
\usepackage{comment}

	\newtheoremstyle{myplain}
	  {}
	  {}
	  {\itshape}
	  {}
	  {\bfseries}
	  {}
	  {5pt plus 1pt minus 1pt}
	  {}

	\newtheoremstyle{mydefinition}
	  {}
	  {}
	  {\normalfont}
	  {}
	  {\bfseries}
	  {}
	  {5pt plus 1pt minus 1pt}
	  {}
  
	\theoremstyle{myplain}

	\newtheorem{theorem}{Theorem}
	\newtheorem{lemma}{Lemma}
	\newtheorem{proposition}{Proposition}

	\theoremstyle{mydefinition}
	\newtheorem{definition}{Definition}

    \newcommand{\prl}[1] 		{\left(#1\right)}
    \newcommand{\brl}[1] 		{\left[ #1 \right]}
    \newcommand{\crl}[1] 		{\left\{#1\right\}}
    
    \newcommand{\R}             {\mathbb{R}} 
 	\newcommand{\N}         		{\mathbb{N}} 
    
    \newcommand{\atan} 			{\mathrm{atan}}

  	\newcommand{\card}[1]		{\left| #1 \right|} 
	\newcommand{\norm}[1] 		{\left \| #1 \right \|} 
    \newcommand{\absval}[1] 		{\left | #1 \right |}
    
    \newcommand{\vectprod}[2]    { \tr{#1} #2}
    \newcommand{\tr}[1]         {{#1}^\mathrm{T}}
     
	\newcommand{\argmax}{\operatornamewithlimits{arg\ max}} 

	\newcommand{\mat}[1] 		{\mathrm{\mathbf{#1}}}

	\newcommand{\vect}[1] 		{\mathrm{#1}}
	\newcommand{\vectbf}[1]		{\mathrm{\mathbf{#1}}}

    \newcommand{\twovec} [2] { \left[ \begin{array}{c} 
          #1 \\ #2 \end{array} \right] }

    \newcommand{\refeqn}[1]			{(\ref{#1})}
    \newcommand{\reffig}[1]			{Fig. \ref{#1}}
    \newcommand{\refsec}[1]			{Section \ref{#1}}

	\newcommand{\refthm}[1]			{Theorem \ref{#1}}
	\newcommand{\refprop}[1]		{Proposition \ref{#1}}
	
	\newcommand{\reflem}[1]			{Lemma \ref{#1}}
	\newcommand{\refdef}[1]			{Definition \ref{#1}}

    \newcommand{\ldf}				{  \: {\mathbf :\! = }  \: }

    \newcommand{\sqz}[1]            {\! #1 \!}

\newcommand{\workspace} 		{Q}
\newcommand{\goal} 			{G}
\newcommand{\stateQ} 		{q}
\newcommand{\stateP} 		{p}
\newcommand{\radius} 		{r}
\newcommand{\powerradius} 	{\rho}
\newcommand{\Sradius}       	{\sigma}
\newcommand{\Bradius}       	{\beta}
\newcommand{\safemargin}     {\epsilon}
\newcommand{\safermargin}    {\varepsilon}
\newcommand{\fevent}			{\phi}
\newcommand{\dist}    		{d}
\newcommand{\fdegradation}	{f}
\newcommand{\Gpart} 			{\mathcal{W}}
\newcommand{\Vpart} 			{\mathcal{V}}
\newcommand{\Ppart} 			{\mathcal{P}}

\newcommand{\Fpart} 			{\mathcal{F}}
\newcommand{\Apart}         	{\mathcal{A}}
\newcommand{\Tpart}         	{\mathcal{T}}
\newcommand{\Spart}         	{\mathcal{S}}
\newcommand{\Bpart}         	{\mathcal{B}}
\newcommand{\Gcell} 			{W}
\newcommand{\Vcell} 			{V}
\newcommand{\Pcell} 			{P}

\newcommand{\Acell}         	{A}
\newcommand{\Fcell} 			{F}
\newcommand{\Tcell}			{T}
\newcommand{\Scell}        	{S}
\newcommand{\Bcell}         	{B}
\newcommand{\flocopt}		{\mathcal{H}}
\newcommand{\mass} 			{m}
\newcommand{\ctrd} 			{\vect{c}}
\newcommand{\ctrlinput} 		{u}
\newcommand{\ctrlgain}  		{k}

\newcommand{\disk}           {D} 
\newcommand{\confspace}      {\mathrm{Conf}}
\newcommand{\ctrlvel}        {v}
\newcommand{\ctrlang}        {\omega}

\newcommand{\localdomain}    {\mathcal{D}}
\newcommand{\conv}           {\mathrm{conv}}

\title{\LARGE \bf 
Voronoi-Based Coverage Control of Heterogeneous Disk-Shaped Robots
}

\author{Omur Arslan  and Daniel E. Koditschek%
\thanks{The authors are with the Department of Electrical and Systems Engineering, University of Pennsylvania, Philadelphia, PA 19104. 
E-mail: \{omur, kod\}@seas.upenn.edu. This work was supported by AFOSR under the CHASE MURI FA9550-10-1-0567.}%
}

\begin{document}

\maketitle
\thispagestyle{empty}
\pagestyle{empty}

\begin{abstract}

In distributed mobile sensing applications, networks of agents that are heterogeneous respecting both actuation as well as body and sensory footprint are often modelled by recourse to power diagrams --- generalized Voronoi diagrams with additive weights. 
In this paper we adapt the body power diagram to introduce its ``free subdiagram,'' generating a vector field planner that solves the combined sensory coverage and collision avoidance problem via continuous evaluation of an associated constrained optimization problem.  
We propose practical extensions (a heuristic congestion manager that speeds convergence and a lift of the point particle controller to the more practical differential drive kinematics) that  maintain the convergence and collision guarantees.
\end{abstract}

\section{INTRODUCTION}

Among the many proposed multiple mobile sensor coordination strategies \cite{schwager_rus_slotine_IJRR2011},  Voronoi-based coverage control \cite{cortes_martinez_karatas_bullo_TRA2004} uniquely combines  both deployment and allocation in an intrinsically  distributed manner \cite{okabe_etal_2009} via gradient descent (the ``move-to-centroid'' law) down a utility function minimizing the expected event sensing cost to adaptively achieve
a \emph{centroidal Voronoi configuration} (depicted on the left in  \reffig{fig.VoronoiPowerDiagram}).   
Since the original application to homogeneous point robots \cite{cortes_martinez_karatas_bullo_TRA2004}, a growing 
literature considers the extension to heterogeneous groups of robots differing variously in their sensorimotor capabilities \cite{pimenta_kumar_mesquita_pereira_CDC2008, kantaros_thanou_tzes_Automatica2015, pierson_etal_ICRA2015, kwok_martinez_IJRNC2010} by recourse to \emph{power diagrams} --- generalized Voronoi diagrams with additive weights \cite{aurenhammer_JoC1987}.

\subsection{Motivation and Prior Literature}

Although it inherits many nice properties of a standard Voronoi diagram such as convexity and  dual triangulability, a power diagram may possibly have empty cells associated with some (unassigned) robots and/or some robots may not be contained in their nonempty cells  \cite{aurenhammer_JoC1987}, as situation depicted on the middle in \reffig{fig.VoronoiPowerDiagram}. 
Such   \emph{occupancy defects} (\refdef{def.OccupancyDefect})  generally cost resource inefficiency or redundancy\footnote{Note that a power diagram with an occupancy defect can be beneficial in certain applications to save/balance energy across a mobile network of power limited agents\cite{kwok_martinez_IJRNC2010}.}, and, crucially,  they re-introduce the problem of collision avoidance --- the chief motivation for the present paper.

\begin{figure}[t]
\centering
\begin{tabular}{@{}c@{\hspace{0.5mm}}c@{\hspace{0.5mm}}c@{}}
\includegraphics[width=0.16\textwidth]{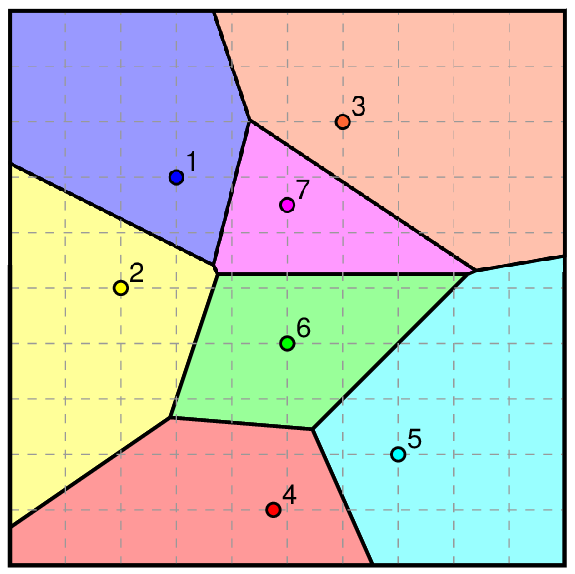} 
&
\includegraphics[width=0.16\textwidth]{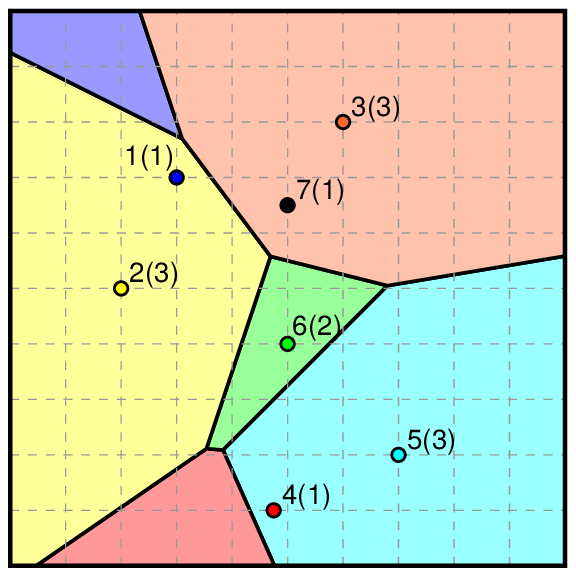} 
&
\includegraphics[width=0.16\textwidth]{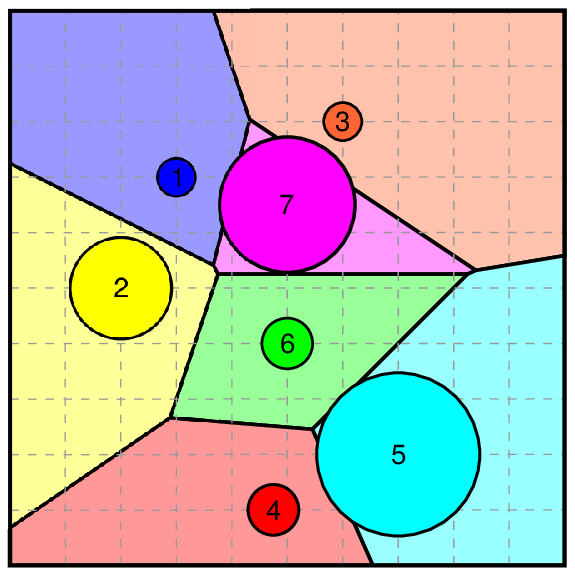} 
\end{tabular}
\vspace{-1mm}
\caption{An illustration of (left) the Voronoi  and (middle) power  diagrams of an environment based on a noncolliding placement of point robots, where the weights of power cells are shown in parentheses. 
Although each point robot is always contained in its Voronoi cell, power cells associated with some robots (e.g. the 7th robot) may be empty  and/or some robots (e.g. the 1st and 4th robots) may not be contained in their nonempty power cells.
(Right) A collision free disk configuration does not necessarily have Voronoi cells containing respective robot bodies.
}
\label{fig.VoronoiPowerDiagram}
\vspace{-2mm}
\end{figure}

Voronoi-based coverage control implicitly entails collision avoidance for point robots since robots move in their pairwise disjoint Voronoi cells \cite{cortes_martinez_karatas_bullo_TRA2004}, but an additional collision avoidance strategy is mandatory for safe navigation of finite size robots.  
Existing work on combining coverage control and collision avoidance generally uses  (i) either heuristic approaches based on repulsive fields \cite{dirafzoon_menhaj_afshar_CCA2010, pimenta_etal_WAFR2008} and
reciprocal velocity obstacles \cite{breitenmoser_martinoli_DARS2014} causing robots to converge to configurations far from  optimal sensing configurations; or (ii) the projection of a vector field  whenever a robot reaches the boundary of  its partition cell \cite{pimenta_kumar_mesquita_pereira_CDC2008, kantaros_thanou_tzes_ICRA2014} introducing a source of discontinuity.
An important observation made in \cite{pimenta_kumar_mesquita_pereira_CDC2008} is that it is sufficient to restrict robot bodies to respective Voronoi regions for collision avoidance, but this is a conservative assumption for robot groups with different body sizes (as illustrated on the right in \reffig{fig.VoronoiPowerDiagram}).

\subsection{Contributions and Organization of the Paper}

In this paper, we provide a necessary and sufficient condition for identifying collision free  configurations of finite size robots in terms of their power diagrams, and 
accordingly propose a constrained coverage control (``move-to-constrained-centroid'') law  whose continuous and piecewise smooth flow asymptotically converges to an optimal sensing configuration avoiding any collisions along the way.
We extend  the practicability of the result by adding a congestion management heuristic for unassigned robots that hastens  the assigned robots' progress, and, finally, adapt  the fully actuated  point particle vector field planner to the widely used kinematic differential drive vehicle model (retaining the convergence and collision avoidance guarantees in both extensions).

This paper is organized as follows. 
\refsec{sec.CoverageControlPointRobot}  briefly summarizes coverage control of point robots. 
\refsec{sec.OccupancyDefect} discusses occupancy defects of power diagrams.
In \refsec{sec.CollisionAvoidance} we introduce  a novel use  of power diagrams for identifying  collision free multirobot configurations, and then propose a constrained optimization framework combining  area coverage and collision avoidance, and present its practical  extensions.
\refsec{sec.NumericalSimulation} offers some numerical studies of
the proposed algorithms. 
\refsec{sec.Conclusion} concludes with a summary of our contributions and a brief discussion of future work.

\section{Coverage Control of Point Robots}
\label{sec.CoverageControlPointRobot}

\subsection{Location Optimization of Homogeneous Robots}
\label{sec.CoverageControlHomogeneous}

Let $\workspace$ be a convex environment in $\R^N$ with a priori known event distribution function $\fevent:Q \rightarrow \R_{>0}$ that models the probability of some event occurs in $\workspace$, and $\vectbf{\stateP} \ldf \prl{\vect{\stateP_1}, \vect{\stateP_2}, \ldots, \vect{\stateP_n}} \in \workspace^{n}$  be a (noncolliding) placement of $n \in \N$ point robots in $\workspace$.\footnote{Here, $\N$ is the set of all natural numbers; $\R$ and $\R_{>0}$ ($\R_{\geq0}$) denote the set of real and positive (nonnegative) real numbers, respectively; and $\R^N$ is the $N$-dimensional Euclidean space.}
Suppose that the event detection (sensing) cost of $i$th robot at location $\stateQ \in \workspace$ is a nondecreasing differentiable function, $\fdegradation:\R \rightarrow \R$, of the Euclidean distance, $\norm{\vect{\stateQ}-\vect{\stateP}_i}$, between $\vect{\stateQ}$ and $\vect{\stateP}_i$.
Further assume that robots are assigned to events based on a partition of $\workspace$  yielding a cover, $\Gpart \ldf \crl{\Gcell_1, \Gcell_2, \ldots, \Gcell_n}$, a collection of  subsets (``cells''), $\Gcell_i$,  whose union returns $\workspace$ but whose cells have mutually disjoint interiors.~\footnote{We will  generally refer to such decompositions as ``diagrams'' but also occasionally allow the slight abuse of language to follow tradition and refer to $\Gpart$ as a \emph{partition}.}
A well established approach (arising in both facility location \cite{hamacher_drezner_2002,okabe_etal_2009} and quantization \cite{lloyd_IT1982,qu_faber_gunzburger_SIAMReview1999} problems) achieves such a cover by minimizing the expected event sensing cost,
\begin{equation}\label{eq.flocopt_homogeneous}
\flocopt\prl{\vectbf{\stateP}, \Gpart} \ldf \sum_{i = 1}^{n} \int_{\Gcell_i} \fdegradation\prl{\norm{\vect{\stateQ}- \vect{\stateP}_i}}\fevent\prl{\vect{\stateQ}}\vect{d}\vect{q}.
\end{equation}

Now observe that, for any fixed $\vectbf{\stateP}$, the optimal task assignment minimizing  $\flocopt$ is the standard Voronoi diagram $\Vpart\prl{\vectbf{\stateP}} \ldf \crl{\Vcell_1, \ldots, \Vcell_n}$ of $\workspace$ based on the configuration $\vectbf{\stateP}$,
\begin{equation}\label{eq.Vcell}
\Vcell_i = \crl{\vect{\stateQ} \in \workspace ~\Big| \norm{\vect{\stateQ} - \vect{\stateP}_i} \leq \norm{\vect{\stateQ} - \vect{\stateP}_j}, \forall j \neq i }.
\end{equation} 
%
%
Thus, given the optimal task assignment of robots, the objective function $\flocopt$ takes the following form
\begin{equation}
\flocopt_{\Vpart}\prl{\vectbf{\stateP}} \!\ldf \flocopt\prl{\vectbf{\stateP}, \Vpart\prl{\vectbf{\stateP}}} = \sum_{i=1}^n \int_{\Vpart_i} \!\! \fdegradation\prl{\norm{\vect{\stateQ} - \vect{\stateP}_i}}\fevent\prl{\vect{\stateQ}} \vect{d}\vect{\stateQ}, \!\!
\end{equation} 
and it is common knowledge that  \cite{qu_faber_gunzburger_SIAMReview1999,okabe_etal_2009, cortes_martinez_karatas_bullo_TRA2004} 
\begin{equation}
\frac{\partial\flocopt_{\Vpart}\prl{\vectbf{\stateP}}}{\partial \vect{\stateP}_i} = \int_{\Vcell_i} \frac{\partial}{\partial\vect{\stateP}_i}\fdegradation\prl{\norm{\vect{\stateQ} - \vect{\stateP}_i}}\fevent\prl{\vect{\stateQ}} \vect{d}\vect{\stateQ}. 
\end{equation}
In the special case of $\fdegradation\prl{x} = x^2$, the partial derivative of $\flocopt_{\Vpart}$ has a simple physical interpretation as follows:
\begin{align}\label{eq.HomogeneousGradient}
\frac{\partial\flocopt_{\Vpart}\prl{\vectbf{\stateP}}}{\partial \vect{\stateP}_i} = 2 \mass_{\Vcell_i} \prl{\vect{\stateP}_i -  \ctrd_{\Vcell_i}},
\end{align} 
where $\mass_{\Vcell_i}$ and $\ctrd_{\Vcell_i}$, respectively, denote the mass and the center of mass of $\Vcell_i$ according to the mass density function~$\fevent$, 
\begin{align}\label{eq.mass_ctrd}
\mass_{\Vcell_i} \ldf \int_{\Vcell_i} \fevent\prl{\stateQ}\text{d}\stateQ, \qquad  \ctrd_{\Vcell_i}\ldf \int_{\Vcell_i} \stateQ ~\fevent\prl{\stateQ}\text{d}\stateQ.
\end{align}
%

Assuming first order (completely actuated single integrator) robot dynamics,
\begin{equation} \label{eq.SystemModel}
\dot{\vect{\stateP}}_i = \vect{\ctrlinput}_i, 
\end{equation}
the standard ``move-to-centroid'' law asymptotically steering point robots to a centroidal Voronoi configuration with the guarantee of no collision along the way is 
\begin{align}\label{eq.move2ctrdhomogeneous}
\vect{\ctrlinput}_i = - \ctrlgain \prl{\vect{\stateP}_i - \ctrd_{\Vcell_i}},
\end{align}
where $\ctrlgain \in \R_{> 0}$ is a fixed control gain and the Voronoi diagram $\Vpart\prl{\vectbf{\stateP}}$  of $\workspace$ is assumed to be continuously updated.
Note that $\mass_{\Vcell_i}$ and $\ctrd_{\Vcell_i}$ are both continuously differentiable functions of $\vectbf{\stateP}$ as are both  $\flocopt_{\Vpart}$ and $\vect{\ctrlinput}_i$ \cite{bullo_cortes_martinez_DistributedControlRoboticNetworks_2009}. 
Finally, observe, again, that the coverage control $\vect{\ctrlinput}_i$ supports  a distributed implementation whose local communications structure is specified by the associated Delaunay graph \cite{cortes_martinez_karatas_bullo_TRA2004}.

\subsection{Location Optimization of Heterogeneous Robots}
\label{sec.CoverageControlHeterogeneous}

In distributed sensing applications, heterogeneity of robotic networks in sensing and actuation \cite{pimenta_kumar_mesquita_pereira_CDC2008, kantaros_thanou_tzes_Automatica2015, pierson_etal_ICRA2015, kwok_martinez_IJRNC2010} is often modelled by recourse to \emph{power diagrams}, generalized Voronoi diagrams with additive weights \cite{aurenhammer_JoC1987}.
More precisely, for a given multirobot configuration $\vectbf{\stateP} \in \workspace^n$, the event sensing cost of $i$th robot at location $\vect{\stateQ} \in \workspace$ is assumed to be given by the \emph{power distance}, $\norm{\vect{\stateQ} - \vect{\stateP}_i}^2 - \powerradius_i^2$ where $\powerradius_i \in \R_{\geq 0}$ is the \emph{power radius} of $i$th robot.
Accordingly, the task assignment of robots are determined by the power diagram $\Ppart\prl{\vectbf{\stateP}, \boldsymbol{\powerradius}} \ldf \crl{\Pcell_1, \Pcell_2, \ldots, \Pcell_n}$ of $\workspace$ based on the configuration $\vectbf{\stateP}$ and  the associated power radii $\boldsymbol{\powerradius} \ldf \prl{\powerradius_1, \powerradius_2, \ldots, \powerradius_n}$,
\begin{equation}\label{eq.Pcell}
\Pcell_i \ldf \! \crl{\vect{\stateQ} \in \workspace \Big | \norm{\vect{\stateQ}\sqz{-} \vect{\stateP}_i}^2 \sqz{-} \powerradius_i^2 \leq \norm{\vect{\stateQ} \sqz{-} \vect{\stateP}_j}^2 \sqz{-} \powerradius_j^2, \forall j \neq i }\!, \!\! 
\end{equation}
and the location optimization function  becomes
\begin{equation}\label{eq.flocopt_heterogeneous}
\flocopt_{\Ppart}\prl{\vectbf{\stateP}, \boldsymbol{\powerradius}} = \sum_{i =1}^{n} \int_{\Pcell_i} \prl{\norm{\vect{\stateQ}\sqz{-} \vect{\stateP}_i}^2 \sqz{-} \powerradius_i^2} \fevent\prl{\vect{\stateQ}} \vect{d}\vect{\stateQ}.
\end{equation}
Note that in the special case of $\powerradius_i = \powerradius_j$ for all $i\neq j$ the power diagram $\Ppart\prl{\vectbf{\stateP},\boldsymbol{\powerradius}}$ and the Voronoi diagram $\Vpart\prl{\vectbf{\stateP}}$ of $\workspace$ are identical, i.e. $\Pcell_i = \Vcell_i$. 

Similar to \refeqn{eq.HomogeneousGradient}, for fixed $\boldsymbol{\powerradius}$, the partial derivative of $\flocopt_{\Ppart}$ takes the following simple form \cite{pimenta_kumar_mesquita_pereira_CDC2008, kwok_martinez_IJRNC2010, pimenta_etal_WAFR2008},
\begin{equation}\label{eq.HeterogeneousGradient}
\frac{\partial\flocopt_{\Ppart}\prl{\vectbf{\stateP}, \boldsymbol{\powerradius}}}{\partial \vect{\stateP}_i} = 2 \mass_{\Pcell_i}\prl{\vect{\stateP}_i - \ctrd_{\Pcell_i}},
\end{equation}
where $\mass_{\Pcell_i}$ and $\ctrd_{\Pcell_i}$ are the mass and the center of mass of $\Pcell_i$, respectively, as defined in \refeqn{eq.mass_ctrd}. 
\footnote{\label{ft.ctrdwelldefined} To be well defined we set $\ctrd_{\Pcell_i} = \vect{\stateP}_ i$  whenever $\Pcell_i$ has an empty interior.}
For the single integrator robot model \refeqn{eq.SystemModel}, the standard ``move-to-centroid'' law of heterogeneous robotic networks asymptotically driving robots to a critical point of $\flocopt_{\Ppart}\prl{.,\boldsymbol{\powerradius}}$, where robots are located at the centroids of their respective power cells, is defined as
\begin{equation}\label{eq.move2ctrdheterogeneous}
\vect{\ctrlinput}_i = -\ctrlgain\prl{\vect{\stateP}_i - \ctrd_{\Pcell_i}},
\end{equation}   
for some fixed $\ctrlgain \in \R_{> 0}$ and the power diagram $\Ppart\prl{\vectbf{\stateP}, \boldsymbol{\powerradius}}$ of $\workspace$ is assumed to be continuously updated.
Notwithstanding its welcome inheritance of many standard Voronoi properties (e.g.,  convexity, dual triangulability),  a power diagram may yield  empty cells associated with some robots and/or some robots may not be contained in their nonempty power cells, illustrated in \reffig{fig.VoronoiPowerDiagram}. 
In consequence, contrary to the case of homogeneous robots, the ``move-to-centroid'' law of heterogeneous point robots is discontinuous and it cannot guarantee collision free navigation.  
Thus, in past literature, for robots of finite but heterogeneous size, the standard ``move-to-centroid'' law inevitably requires an additional heuristic collision avoidance strategy for safe navigation.

\section{Occupancy Defects of Power Diagrams}
\label{sec.OccupancyDefect}

%
\begin{definition}\label{def.OccupancyDefect}
(\emph{Occupancy Defect}) The \emph{power partition}, $\Ppart\prl{\vectbf{\stateP}, \boldsymbol{\powerradius}}$, associated with configuration $\vectbf{\stateP} \in \workspace^n$ and  radii $\boldsymbol{\powerradius} \in \prl{\R_{\geq 0}}^n$ is said to have an \emph{occupancy defect} if $\vect{\stateP}_i \not \in \Pcell_{i}$ for some $i \in \crl{1,2, \ldots n}$. 
\end{definition}

Configurations incurring occupancy defects introduce a number of problems.
First of all, empty partition cells cause resource redundancy because some robots may never be assigned to any event happening around them. 
Such robots  do not only become redundant, but also complicate collision avoidance as  (moving or stationary) obstacles and limit the mobility of others.
In general, robots that are not contained in their respective cells require an extra care for collision avoidance.

A straightforward characterization of an occupancy defective configuration is:
\footnote{In \cite{kantaros_thanou_tzes_Automatica2015} the authors note the issue of empty power cells and give a similar sufficient condition for each robot to be contained in its power cell.
}
\begin{proposition}\label{prop.PartitionDegeneracy}
Given radii $\boldsymbol{\powerradius} \in \prl{\R_{\geq 0}}^n$, configuration $\vectbf{\stateP} \in \workspace^n$ does not incur an occupancy defective power diagram  if and only if $\norm{\vect{\stateP}_i - \vect{\stateP}_j}^2\geq \absval{\powerradius_i^2 - \powerradius_j^2}$ for all $i \neq j$. 
\end{proposition} 
\begin{proof}
By \refdef{def.OccupancyDefect}, $\Ppart\prl{\vectbf{\stateP}, \boldsymbol{\powerradius}}$ has no occupancy defect if and only if $\vect{\stateP}_i \in \Pcell_i$ for all $i$, which is the case if and only if
%
%
\begin{align}
\norm{\vect{\stateP}_i - \vect{\stateP}_i}^2  - \powerradius_i^2 &\leq \norm{\vect{\stateP}_i - \vect{\stateP}_j}^2  - \powerradius_j^2, \\
\norm{\vect{\stateP}_j - \vect{\stateP}_j}^2  - \powerradius_j^2 &\leq \norm{\vect{\stateP}_j - \vect{\stateP}_i}^2  - \powerradius_i^2, 
\end{align}
for all $i \neq j$. 
Thus, the result follows.
\end{proof}


\section{Combining Coverage Control and Collision Avoidance}
\label{sec.CollisionAvoidance}

Throughout the rest of paper, we consider heterogeneous disk-shaped multirobot configurations, $\vectbf{\stateP} = \prl{\vect{\stateP}_1, \vect{\stateP}_2, \ldots, \vect{\stateP}_n} \in \workspace^{n}$, in $\workspace$  with associated vectors of nonnegative body radii $\boldsymbol{\Bradius}\ldf \prl{\Bradius_1, \Bradius_2, \ldots, \Bradius_n} \in \prl{\R_{\geq 0}}^n$ and  sensory footprint radii $\boldsymbol{\Sradius}\ldf \prl{\Sradius_1, \Sradius_2, \ldots, \Sradius_n} \in \prl{\R_{\geq 0}}^n$, where $i$th robot is centered at $\vect{\stateP}_i \in Q$ and has body radius $\Bradius_i \geq 0$ and sensory footprint radius $\Sradius_i \geq 0$.
Accordingly, we will denote by  $\Bpart\prl{\vectbf{\stateP}, \boldsymbol{\Bradius}} = \crl{\Bcell_1, \Bcell_2, \ldots, \Bcell_n}$, a cover we term   
the \emph{body diagram} of $\workspace$,  solving the power problem \refeqn{eq.Pcell}, \refeqn{eq.flocopt_heterogeneous}, defined  from $\flocopt_{\Bpart}\prl{\vectbf{\stateP}, \boldsymbol{\Bradius}}$; and we will denote by  $\Spart\prl{\vectbf{\stateP}, \boldsymbol{\Sradius}} = \crl{\Scell_1, \Scell_2, \ldots, \Scell_n}$, a cover we term the sensor diagram of $\workspace$, solving the corresponding problem defined  by  $\flocopt_{\Spart}\prl{\vectbf{\stateP}, \boldsymbol{\Sradius}}$. 
We also find it convenient to denote the configuration space of body-noncolliding disks of radii $\boldsymbol{\Bradius}$ in $\workspace$ as
\begin{align}\label{eq.confspace}
\confspace\prl{\workspace, \boldsymbol{\Bradius}} \ldf \left\{ \Big.\right.\!\vectbf{\stateP} \in \workspace^n \Big| & \norm{\vect{\stateP}_i \sqz{-} \vect{\stateP}_j} > \Bradius_i \sqz{+} \Bradius_j ~~ \forall i \neq j,\nonumber \\
& \hspace{5mm}\disk\prl{\vect{\stateP}_i, \Bradius_i} \subset \mathring{Q} \quad \forall i \!\left. \Big.\right \}\!,\!
\end{align}
where $\disk\prl{\vect{x}, \radius} \ldf \crl{\vect{y} \in \R^{N} \big| \norm{\vect{x} - \vect{y}} \leq \radius}$ is the closed disk in $\R^N$ centered at $\vect{x} \in \R^N $ with radius $\radius \geq 0$, and $\mathring{\workspace}$ is the interior of $\workspace$.
Note that the vectors of  body radii $\boldsymbol{\Bradius}$ and sensory footprint radii $\boldsymbol{\Sradius}$ are not necessary equal since $\boldsymbol{\Bradius}$ models the heterogeneity of robots in body size, $\boldsymbol{\Sradius}$ models their heterogeneity in sensing and actuation.

\subsection{Encoding Collisions via Body Diagrams}

A geometric characterization of collision free multirobot configurations in $\workspace$ via their body diagrams is:
\begin{proposition}\label{prop.CollisionPowerDiagram}
Let $\Bpart\prl{\vectbf{\stateP}, \boldsymbol{\Bradius}}$ be the body diagram of $\workspace$ associated with configuration  $\vectbf{\stateP} \in  \workspace^{n}$ and  body radii $\boldsymbol{\Bradius} \in \prl{\R_{\geq 0}}^n$. 
Then  $\vectbf{\stateP}$ is  collision free  if and only if every robot body is contained in the interior of its body cell, i.e.
\begin{equation}
\!\vectbf{\stateP} \in \confspace\prl{\workspace,\boldsymbol{\Bradius}} \Longleftrightarrow \disk\prl{\vect{\stateP}_i, \Bradius_i} \subset \mathring{\Bcell}_i \quad  \forall i .\!\!\!
\end{equation}
\end{proposition}
\begin{proof}
The sufficiency ($\Longleftarrow$) follows because  $\Bpart\prl{\vectbf{\stateP}, \boldsymbol{\Bradius}}$ is a cover of $\workspace$ whose elements have disjoint interiors.
Hence, given $\disk\prl{\vect{\stateP}_i, \Bradius_i} \subset \mathring{\Bcell}_i$ for all $i$, we have  $\disk\prl{\vect{\stateP}_i, \Bradius_i} \subset \mathring{\workspace}$ and $\disk\prl{\vect{\stateP}_i, \Bradius_i} \cap \disk\prl{\vect{\stateP}_j, \Bradius_j} = \emptyset$ for all $i \neq j$, and so $\norm{\vect{\stateP}_i - \vect{\stateP}_j} > \Bradius_i + \Bradius_j$.
Thus, $\vectbf{\stateP} \in \confspace\prl{\workspace,\boldsymbol{\Bradius}}$.

To see the necessity ($\Longrightarrow$), for any $\vectbf{\stateP} \in \confspace\prl{\workspace,\boldsymbol{\Bradius}}$ we will show that $\vect{\stateP}_i \in \Bcell_i$ for all $i$, and the distance between $\vect{\stateP}_i$ and the boundary $\partial \Bcell_i$ of $\Bcell_i$ is greater than $\Bradius_i$, i.e. $\min_{\vect{x}\in \partial \Bcell_i} \norm{\vect{x} - \vect{\stateP}_i} > \Bradius_i$, and so $\disk\prl{\vect{\stateP}_i, \Bradius_i} \subset \mathring{\Bcell}_i$.

It follows from \refprop{prop.PartitionDegeneracy} that for any $\vectbf{\stateP} \in \confspace\prl{\workspace,\boldsymbol{\Bradius}}$ ~   $\Bpart\prl{\vectbf{\stateP}, \boldsymbol{\Bradius}}$ has no occupancy defect (Def. \ref{def.OccupancyDefect}), i.e. $\vect{\stateP}_i \in \Bcell_i ~~\forall i$.

The boundary $\partial \Bcell_i$ of $\Bcell_i$ is defined by the boundary $\partial \workspace$ of $\workspace$ and  the separating separating hyperplane between body cells $\Bcell_i$ and $\Bcell_j$ for some $j \neq i$ \cite{aurenhammer_JoC1987}. 
By definition \refeqn{eq.confspace}, we have $\min_{\vect{x}\in \partial \workspace} \norm{\vect{x} - \vect{\stateP}_i}> \Bradius_i$ for any $\vectbf{\stateP} \in \confspace\prl{\workspace, \boldsymbol{\Bradius}}$.

Now observe that, for any $i\neq j$ the separating hyperplane between body cells $\Bcell_i$ and $\Bcell_j$ is perpendicular to the line joining $\vect{\stateP}_i$ and $\vect{\stateP}_j$  and is given by \cite{aurenhammer_JoC1987}
\begin{equation}
\!\!\!H_{ij} \ldf\!\!\crl{\vect{x} \sqz{\in} \R^{N} \Big| 2 \vectprod{\vect{x}}{\!\!\prl{\vect{\stateP}_i \sqz{-} \vect{\stateP}_j} \sqz{=} \Bradius_j^2 \sqz{-} \Bradius_i^2 \sqz{+} \norm{\vect{\stateP}_i}^2 \sqz{-} \norm{\vect{\stateP}_j}}^2 \!}\!,\!\!   \!
\end{equation}
and the perpendicular distance of $\vect{\stateP}_i$ to $H_{ij}$ is given by
\begin{align}
\dist\prl{\vect{\stateP}_i, H_{ij}} \ldf \frac{\norm{\vect{\stateP}_i - \vect{\stateP}_j}}{2} + \frac{\Bradius_i^2 - \Bradius_j^2}{2\norm{\vect{\stateP}_i - \vect{\stateP}_j}}.
\end{align}
Note that $\dist\prl{\vect{\stateP}_i, H_{ij}}$ is negative when $\Bpart\prl{\vectbf{\stateP}, \boldsymbol{\Bradius}}$ has an occupancy defect; and we have from \refprop{prop.PartitionDegeneracy} that $\Bpart\prl{\vectbf{\stateP}, \boldsymbol{\Bradius}}$ is free of such a defect for any $\vectbf{\stateP} \in \confspace\prl{\workspace,\boldsymbol{\Bradius}}$ and so $\dist\prl{\vect{\stateP}_i, H_{ij}} \geq 0$.
One can further show that for any $i \neq j$
\begin{align}
\dist\prl{\vect{\stateP}_i, H_{ij}} &= \Bradius_i + \frac{\norm{\vect{\stateP}_i-\vect{\stateP}_i}^2 + \Bradius_i^2 - \Bradius_j^2 - 2 \Bradius_i\norm{\vect{\stateP}_i-\vect{\stateP}_i}}{2\norm{\vect{\stateP}_i-\vect{\stateP}_i}}, \nonumber \\
& = \Bradius_i + \underbrace{\frac{\prl{\norm{\vect{\stateP}_i-\vect{\stateP}_i} - \Bradius_i}^2 - \Bradius_j^2}{2\norm{\vect{\stateP}_i-\vect{\stateP}_i}}}_{> 0, \text{ since } \vectbf{\stateP} \in \confspace\prl{\workspace,\boldsymbol{\Bradius}} } 
> \Bradius_i, 
\end{align}
which completes the proof.
\end{proof}

To determine a collision free neighborhood of a configuration $\vectbf{\stateP} \sqz{\in} \confspace\!\prl{\workspace, \boldsymbol{\Bradius}}$ with a vector of body radii $\boldsymbol{\Bradius} \sqz{\in} \prl{\R_{\geq0}}^n$, we define a \emph{free subdiagram} $\Fpart\prl{\vectbf{\stateP}, \boldsymbol{\Bradius}} \ldf \crl{\Fcell_1, \Fcell_2, \ldots, \Fcell_n}$  of the body diagram $\Bpart\prl{\vectbf{\stateP}, \boldsymbol{\Bradius}} = \crl{\Bcell_1, \Bcell_2, \ldots, \Bcell_n}$ by eroding each cell removing the volume swept along its boundary, $\partial \Bcell_i$, by the associated body radius, see \reffig{fig.PowerDiagramCollisionFree}, as \cite{haralick_sternberg_zhuang_PAMI1987}~\footnote{Here, $\vectbf{0}$ is a vector of all zeros with the appropriate size, and  $A \oplus B \ldf \crl{a + b ~|~ a \in A, b \in B }$ is the Minkowski sum of sets $A$ and $B$.}
\begin{equation}\label{eq.Fcell}
\!\!\!\Fcell_i \ldf \Bcell_i \setminus \! \prl{\big.\partial \Bcell_i \sqz{\oplus} \disk\!\prl{\vectbf{0}, \Bradius_i}\!} \sqz{=}  \crl{\!\vect{\stateQ} \sqz{\in} \Bcell_i \bigg | \!\min_{\vect{x} \in \partial \Bcell_i}\!\!\norm{\vect{x} \sqz{-} \vect{\stateQ}} \sqz{>} \Bradius_i\!}\!.\!\!\!
\end{equation}  
Note that  $\Fcell_i$ is a nonempty convex set because $\vect{\stateP}_i \in \Fcell_i$ and the erosion of a convex set by a ball is convex. \footnotemark

\footnotetext{It is obvious that the erosion of a half-space by a ball is a half-space. Hence, since the erosion operation is distributed over set intersection \cite{haralick_sternberg_zhuang_PAMI1987}, and  a convex set can be defined as (possibly infinite) intersection of half-spaces \cite{boyd_vandenberghe_ConvexOptimization_2004}, the erosion of a convex set by a ball is convex.}

\begin{figure}[t]
\centering
\begin{tabular}{cc}
\includegraphics[width=0.2\textwidth]{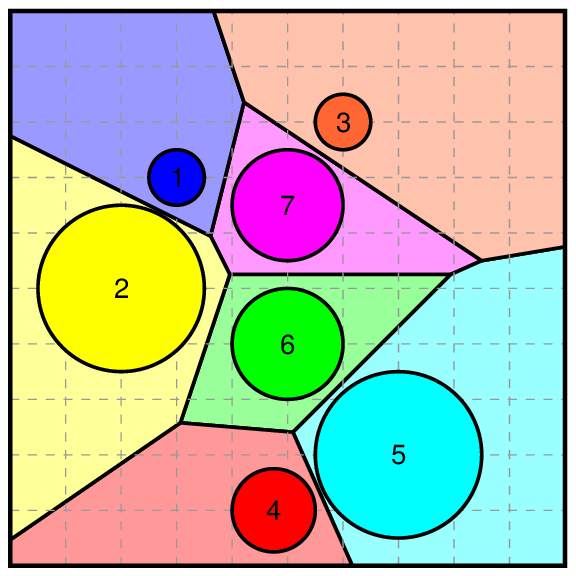} 
&
\includegraphics[width=0.2\textwidth]{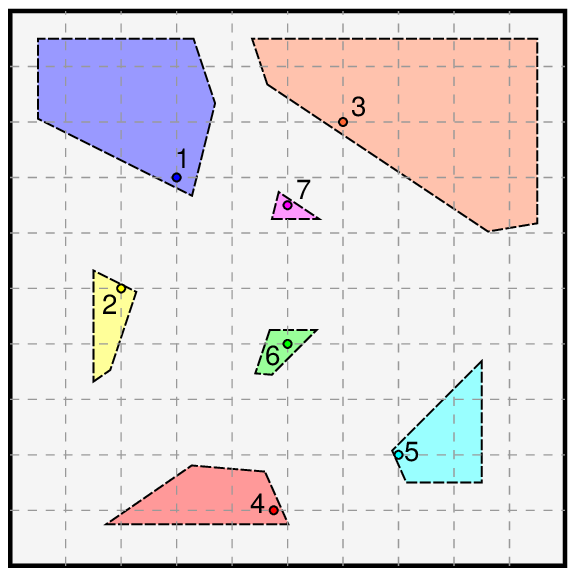} 
\end{tabular}
\vspace{-1mm}
\caption{(left) Encoding collision free configurations via body diagrams: A configuration of disks is nonintersecting iff each disk is contained in the interior of its body cell.
(right) Free subcells, obtained by eroding each body cell with its associated disk radius.}
\label{fig.PowerDiagramCollisionFree}
\end{figure}

The following observation yields a (possibly conservative) convex inner approximation 
of the free configuration space neighborhood surrounding free configuration as
\begin{equation}
\vectbf{\stateP} \in \confspace\prl{\workspace, \boldsymbol{\Bradius}} \Rightarrow 
\prod \Fpart\prl{\vectbf{\stateP}, \boldsymbol{\Bradius}} 
\subset \confspace\prl{\workspace, \boldsymbol{\Bradius}},
\end{equation}
where $\prod \Fpart\prl{\vectbf{\stateP}, \boldsymbol{\Bradius}} = \Fcell_1 \times \Fcell_2 \times \hdots \times \Fcell_n$.

\begin{lemma}\label{lem.LocalCollisionFreeZone}
Let $\vectbf{\stateP} \in \confspace\prl{\workspace, \boldsymbol{\Bradius}}$ be
a multirobot  configuration with a vector of body radii $\boldsymbol{\Bradius} \in \prl{\R_{\geq 0}}^n$, and $\Fpart\prl{\vectbf{\stateP}, \boldsymbol{\Bradius}}$ be the free subdiagram of the body diagram $\Bpart\prl{\vectbf{\stateP}, \boldsymbol{\Bradius}}$.

Then  $\vectbf{\stateQ} \in \workspace^n$ is a collision free  multirobot configuration in $\confspace\prl{\workspace, \boldsymbol{\Bradius}}$ if $\vect{\stateQ}_i \in \Fcell_i$ (i.e. $\disk\prl{\vect{\stateQ}_i, \Bradius_i} \subset \mathring{\Bcell}_i$) for all $i$.
\end{lemma}
\begin{proof}
The results directly follows from $\Bpart\prl{\vectbf{\stateP}, \boldsymbol{\Bradius}}$ covering~a partition of $\workspace$, as discussed in the proof of  \refprop{prop.CollisionPowerDiagram}. 
\end{proof}

\subsection{Coverage Control of Heterogeneous Disk-Shaped Robots}
\label{sec.move-to-constrained-centroid}


Consider a heterogeneous multirobot configuration $\vectbf{\stateP} \in \confspace \prl{\workspace, \boldsymbol{\Bradius} \sqz{+} \boldsymbol{\safemargin}}$ with associated vectors of body radii $\boldsymbol{\Bradius} \in \prl{\R_{\geq 0}}^n$, safety margins $\boldsymbol{\safemargin} \in \prl{\R_{> 0}}^n$ and sensory footprint radii $\boldsymbol{\Sradius} \in \prl{\R_{\geq 0}}^n$, and let $\Spart\prl{\vectbf{\stateP}, \boldsymbol{\Sradius}} = \crl{\Scell_1, \ldots, \Scell_n}$  be the sensory diagram of $\workspace$ based on robot locations $\vectbf{\stateP}$ and sensory footprint radii $\boldsymbol{\Sradius}$, and $\Fpart\prl{\vectbf{\stateP}, \boldsymbol{\Bradius} \sqz{+} \boldsymbol{\safemargin}} = \crl{\Fcell_1, \ldots, \Fcell_n}$ be the free subdiagram  associated with  configuration  $\vectbf{\stateP}$ and enlarged body radii $\boldsymbol{\Bradius} \sqz{+}\boldsymbol{\safemargin}$.
Here we use $\boldsymbol{\safemargin}$ to guarantee the clearance between any pair $i \neq j$ of robots to be at least $\safemargin_i + \safemargin_j$. \footnotemark

\footnotetext{Having a positive vector of safety margins $\boldsymbol{\safemargin}$ enables us to consider collision free configurations in $\overline{\confspace\prl{\workspace, \boldsymbol{\Bradius} \sqz{+} \boldsymbol{\safemargin}}} \subset \confspace\prl{\workspace, \boldsymbol{\Bradius}}$. 
Throughout the rest of the paper, in order the compress the notation,  we will abuse the notation and use  $\confspace\prl{\workspace, \boldsymbol{\Bradius} \sqz{+} \boldsymbol{\safemargin}}$ to refer to the closure of the configuration space in \refeqn{eq.confspace}.}

Now, in contrast to the standard ``move-to-centroid'' law that steers each robot directly towards the centroid, $\ctrd_{\Scell_i}$,  of its sensory cell, $\Scell_i$, we propose a coverage control policy that selects a safe target location, called the \emph{constrained centroid} of $\Scell_i$, that solves the following convex programming~\footnote{Here, $\overline{A}$ is the closure of set $A$.}
\begin{equation}\label{eq.ConstrainedOptimization}
\begin{split}
\text{minimize}  &\quad \norm{\vect{\stateQ}_i - \ctrd_{\Scell_i}}^2 \\
\text{subject to} &\quad  \vect{\stateQ}_i \in \overline{\Fcell}_i
\end{split}
\end{equation}
where $\overline{\Fcell}_i$ is a closed convex set.
It is well known that the unique solution of \refeqn{eq.ConstrainedOptimization} is given by \cite[Section 8.1.1]{boyd_vandenberghe_ConvexOptimization_2004} \footnote{In general, the metric projection of a point onto a convex set can be efficiently computed using a standard convex programming solver \cite{boyd_vandenberghe_ConvexOptimization_2004}.
If $\workspace$ is a convex polytope, then a free subcell, $\Fcell_i$, is also a convex polytope and can be written as a finite intersection of half-spaces.
Hence, the metric projection onto a convex polytope  can be recast as quadratic programming and can be solved in polynomial time \cite{kozlov_tarasov_khachiyan1980}.
In the case of a convex polygonal environment, $\Fcell_i$ is a  convex polygon and the metric projection onto a convex polygon can be solved analytically since the solution lies on one of its edges unless the input point is inside the polygon.
}
\begin{align}\label{eq.ConstrainedCtrd}
\overline{\ctrd}_{\Scell_i} \ldf \left \{ 
\begin{array}{ll}
\ctrd_{\Scell_i} & \text{, if }  \ctrd_{\Scell_i} \in  \overline{\Fcell}_i, \\
\Pi_{\overline{\Fcell}_i}\prl{\ctrd_{\Scell_i}} & \text{, otherwise,}
\end{array}
\right.
\end{align}
where $\Pi_{C}\prl{\vect{x}}$ denotes the metric projection of $\vect{x} \in \R^N$ onto a convex set $C \subset \R^N$, and note that $\Pi_{C}$ is piecewise continuously differentiable \cite{kuntz_scholtes_JMAA1994}.\footnote{
Note that $\ctrd_{\Scell_i}$ is well defined (see footnote \ref{ft.ctrdwelldefined}), hence $\overline{\ctrd}_{\Scell_i}$ must be as well given $\Fcell_i \neq \emptyset$.}  
Accordingly, for the single integrator robot dynamics \refeqn{eq.SystemModel}, our ``move-to-constrained-centroid'' law is defined as
\begin{equation} \label{eq.move2ctrdcollision}
\vect{\ctrlinput}_i = - \ctrlgain\prl{\vect{\stateP}_i - \overline{\ctrd}_{\Scell_i}},
\end{equation}
where $\ctrlgain \in \R_{> 0}$ is a fixed control gain, and we assume that $\Spart\prl{\vectbf{\stateP}, \boldsymbol{\Sradius}}$ and $\Fpart\prl{\vectbf{\stateP}, \boldsymbol{\Bradius} \sqz{+}\boldsymbol{\safemargin}}$ are continuously updated.
We find it convenient to have  $\goal_{\Spart}\prl{\workspace, \boldsymbol{\Bradius} \sqz{+} \boldsymbol{\safemargin}, \boldsymbol{\Sradius}}$ denote the set of equilibria of our ``move-to-constrained-centroid'' law where robots are located at the constrained centroid of their respective sensory cells, \footnote{ Note that this set cannot be empty since it contains the minima of a smooth function over a compact set \refeqn{eq.ConstrainedOptimization}.}
\begin{equation}
\!\goal_{\Spart}\!\prl{\workspace, \boldsymbol{\Bradius} \sqz{+} \boldsymbol{\safemargin}, \boldsymbol{\Sradius}} \sqz{\ldf} \crl{ \vectbf{\stateP} \sqz{\in} \confspace\prl{\workspace, \boldsymbol{\Bradius} \sqz{+} \boldsymbol{\safemargin}} \! \Big|  \,\vect{\stateP}_i \sqz{=} \overline{\ctrd}_{\Scell_i} ~ \forall i  }\!. \!\!
\end{equation}
In the special case of identical sensory footprint radii, i.e. $\Sradius_i = \Sradius_j$ for all $i \neq j$, these stationary configurations are called the constrained centroidal Voronoi configurations \cite{du_gunzburger_ju_JOSC2003}. 
Also note that for homogeneous point robots our ``move-to-constrained-centroid'' law in \refeqn{eq.move2ctrdcollision} simplifies back to the standard ``move-to-centroid'' law in \refeqn{eq.move2ctrdhomogeneous}.

We summarize the qualitative properties of our ``move-to-constrained-centroid'' law  as follows:
\begin{theorem}\label{thm.CollisionQualitative}
For any choice of vectors of body radii $\boldsymbol{\Bradius} \in \prl{\R_{\geq 0}}^n$,  safety margin  $\boldsymbol{\safemargin} \sqz{\in} \prl{\R_{> 0}}^n$ and  sensory footprint radii $\boldsymbol{\Sradius} \sqz{\in} \prl{\R_{\geq 0}}^n$, the configuration space of nonintersecting disks $\confspace\!\prl{\workspace, \boldsymbol{\Bradius} \sqz{+} \boldsymbol{\safemargin}}$ \refeqn{eq.confspace} is positive invariant under the ``move-to-constrained-centroid'' law  in \refeqn{eq.move2ctrdcollision} whose unique, continuous and piecewise differentiable flow, starting at any configuration in $\confspace\!\prl{\workspace, \boldsymbol{\Bradius} \sqz{+} \boldsymbol{\safemargin}}$,  asymptotically reaches a locally optimal sensing configuration in $\goal_{\Spart}\!\prl{\workspace, \boldsymbol{\Bradius} \sqz{+} \boldsymbol{\safemargin}, \boldsymbol{\Sradius}}$ while strictly decreasing the utility function $\flocopt_{\Spart}\!\prl{\cdot, \boldsymbol{\Sradius}}$ \refeqn{eq.flocopt_heterogeneous}  along the way.
If an equilibrium in $\goal_{\Spart}\!\prl{\workspace, \boldsymbol{\Bradius} \sqz{+} \boldsymbol{\safemargin}, \boldsymbol{\Sradius}}$ is isolated, then it is locally asymptotically stable.
\end{theorem}

\begin{proof}


The instantaneous "target" in \refeqn{eq.move2ctrdcollision} lies in the closure of the convex inner approximation to the freespace neighborhood of any free configuration,    $\overline{\ctrd}_{\Spart\prl{\vectbf{\stateP}, \boldsymbol{\Sradius}}} \in \prod \Fpart\prl{\vectbf{\stateP}, \boldsymbol{\Bradius} \sqz{+} \boldsymbol{\safemargin}} \subset \confspace\prl{\workspace, \boldsymbol{\Bradius} \sqz{+} \boldsymbol{\safemargin}}$, hence, according to \reflem{lem.LocalCollisionFreeZone}, the configuration space  tangent vector defined by \refeqn{eq.move2ctrdcollision},
$-\ctrlgain\prl{\vectbf{\stateP} - \overline{\ctrd}_{\Spart\prl{\vectbf{\stateP}, \boldsymbol{\Sradius}}}} \in T_{\vectbf{\stateP}} \confspace\prl{\workspace, \boldsymbol{\Bradius}\sqz{+}\boldsymbol{\safemargin}}$, is either interior directed or, at worse, tangent to the boundary of $\prod \Fpart\prl{\vectbf{\stateP}, \boldsymbol{\Bradius}\sqz{+}\boldsymbol{\safemargin}}$. 
Therefore, by construction \refeqn{eq.ConstrainedOptimization}, the ``move-to-constrained-centroid'' law leaves  $\confspace\!\prl{\workspace,\boldsymbol{\Bradius}\sqz{+}\boldsymbol{\safemargin}}$ positively invariant.

The existence, uniqueness and continuity of its flow can be observed using an equivalent hybrid system consisting of a family of piecewise continuously differentiable local vector fields as follows.
Let $\vectbf{\ctrlinput}^I : \localdomain_I \rightarrow \prl{\R^N}^n$ be a local controller associated with a subset $I$ of $\crl{1,2, \ldots, n}$ defined as
\begin{equation}
\vect{\ctrlinput}^I_i = \left \{
\begin{array}{@{}c@{\hspace{1mm}}l@{}}
 - k\prl{\vect{\stateP_i} - \overline{\ctrd}_{\Scell_i}} & \text{, if } i \in I \\
 \vectbf{0} & \text{, otherwise},
\end{array}
\right .
\end{equation}
where  its domain is 
\begin{equation}
\localdomain_{I} \ldf \crl{ \vectbf{\stateP} \sqz{\in} \confspace\prl{\workspace, \boldsymbol{\Bradius} \sqz{+} \boldsymbol{\safemargin}} \! \Big | \, \mathring{\Scell}_i \neq \emptyset \quad \forall i \in I }.
\end{equation}
Note that for a given configuration in its domain, $\localdomain_I$,  a local policy index, $I$, indicates which robots are assigned to sensory cells with nonempty interiors, and so the domains, $\localdomain_I$, of local controllers defines a finite open cover of    $\confspace\prl{\workspace, \boldsymbol{\Bradius} \sqz{+} \boldsymbol{\safemargin}}$.
Hence, since all unassigned robots are stationary under the ``move-to-constrained-centroid'' law and every robot whose sensory cell  has a nonempty interior is assigned to the coverage task,  one can further conclude that these local controllers can be composed using the policy selection strategy, $g:\confspace\prl{\workspace, \boldsymbol{\Bradius} \sqz{+} \boldsymbol{\safemargin}} \rightarrow \mathbb{P}\prl{n}$ maximizing the number of assigned robots,\footnote{Here $\mathbb{P}\prl{n}$ denotes the set of all subsets of $\crl{1,2, \ldots, n}$.}
%
\begin{equation}
g\prl{\vectbf{\stateP}} \ldf \argmax_{\substack{I \subseteq \crl{1,\ldots,n}\\ \vectbf{\stateP} \in \localdomain_I}} \card{I}.
\end{equation}
such that the resulting hybrid vector field is the same as the ``move-to-constrained-centroid'' law in \refeqn{eq.move2ctrdcollision}, i.e. for any $\vectbf{\stateP} \in \confspace\prl{\workspace, \boldsymbol{\Bradius} \sqz{+} \boldsymbol{\safemargin}}$ 
\begin{align}
\vectbf{\ctrlinput}\prl{\vectbf{\stateP}} = \vectbf{\ctrlinput}^{g\prl{\vectbf{\stateP}}}\prl{\vectbf{\stateP}}.
\end{align}
Note that, since a sensory cell with a nonempty interior can not instantaneously appear or disappear under any continuous motion, each time when a local controller is selected by $g$ it steers the robots for a nonzero time.

Now the continuity properties of each local control policy can be observed as follows. 
As in the case of Voronoi diagrams \cite{bullo_cortes_martinez_DistributedControlRoboticNetworks_2009}, we have that the boundary of  a sensory cell with a nonempty interior is a piecewise continuously differentiable function of robot locations, and  its centroid  is continuously differentiable with respect to robot locations. 
Similarly, the boundary of each element of $\Fpart\prl{\vectbf{\stateP}, \boldsymbol{\Bradius}\sqz{+}\boldsymbol{\safemargin}}$ is piecewise continuously differentiable since each free cell is a nonempty erosion of an element of the body diagram  $\Bpart\prl{\vectbf{\stateP}, \boldsymbol{\Bradius}+ \boldsymbol{\safemargin}}$ by a fixed closed ball.  
Hence, one can conclude that each local control policy is piecewise continuously differentiable since metric projections onto convex cells are piecewise continuously differentiable \cite{kuntz_scholtes_JMAA1994} and the composition of piecewise continuously differentiable functions are also piecewise continuously differentiable \cite{chaney_NA1990}.

Therefore, the existence, uniqueness and continuously differentiability of the flow of  each local controller $\vectbf{\ctrlinput}^I$ follow from the Lipschitz continuity of $\vectbf{\ctrlinput}^I$  in its compact domain $\localdomain_I$  since a piecewise continuously differentiable function is also locally Lipschitz on its domain \cite{chaney_NA1990} and a locally Lipschitz function on a compact set is globally Lipschitz on that set \cite{khalil_NonlinearSystems_2001}.
Hence, since their domains define a finite open
cover of $\confspace\prl{\workspace, \boldsymbol{\Bradius} + \boldsymbol{\safemargin}}$, the unique, continuous and piecewise differentiable flow of the ``move-to-constrained-centroid'' law is constructed by piecing together trajectories of these local policies.

Finally, a natural choice of a Lyapunov function for the stability analysis is the continuously differentiable location optimization function $\flocopt_{\Spart}$ \refeqn{eq.flocopt_heterogeneous}, and one can verify from \refeqn{eq.HeterogeneousGradient}, \refeqn{eq.ConstrainedOptimization} and \refeqn{eq.move2ctrdcollision} that for any $\vectbf{\stateP} \in \confspace\prl{\workspace, \boldsymbol{\Bradius} \sqz{+} \boldsymbol{\safemargin}}$ \footnote{$\tr{\mat{A}}$ is the transpose of matrix $\mat{A}$.}
\begin{align}
\dot{\flocopt}_{\Spart}\prl{\vectbf{\stateP}, \boldsymbol{\Sradius}} &= - k \sum_{i =1}^{n}\mass_{\Scell_i}\hspace{-8mm}\underbrace{2\vectprod{\prl{\vect{\stateP}_i - \ctrd_{\Scell_i}}}{\prl{\vect{\stateP}_i - \overline{\ctrd}_{\Scell_i}}}}_{\substack{\geq \norm{\vect{\stateP}_i - \overline{\ctrd}_{\Scell_i}}^2, \\ \text{ since } \vect{\stateP}_i \in \Fcell_i \text{ and } \norm{\vect{\stateP}_i - \ctrd_{\Scell_i}}^2 \geq \norm{\overline{\ctrd}_{\Scell_i} - \ctrd_{\Scell_i}}^2} } \hspace{-7mm}, \\
& \leq - k \sum_{i =1}^{n}\mass_{\Scell_i} \norm{\vect{\stateP}_i - \overline{\ctrd}_{\Scell_i}}^2 \leq 0, 
\end{align} 
%
%
which is equal to $0$ only if  $\vect{\stateP}_i = \overline{\ctrd}_{\Scell_i}$ for all $i$, i.e. $\vectbf{\stateP} \in \goal_{\Spart}\prl{\workspace, \boldsymbol{\Bradius} \sqz{+} \boldsymbol{\safemargin}, \boldsymbol{\Sradius}}$. 
Thus, it follows from LaSalle's Invariance Principle \cite{khalil_NonlinearSystems_2001} that all multirobot configurations in $\confspace\prl{\workspace, \boldsymbol{\Bradius} \sqz{+} \boldsymbol{\safemargin}}$ asymptotically reach $\goal_{\Spart}\prl{\workspace, \boldsymbol{\Bradius} \sqz{+} \boldsymbol{\safemargin}, \boldsymbol{\Sradius}}$.   
If an equilibrium $\vectbf{\stateP}^*$ in $\goal_{\Spart}\prl{\workspace, \boldsymbol{\Bradius} \sqz{+} \boldsymbol{\safemargin}, \boldsymbol{\Sradius}}$ is isolated, then it is guaranteed that $\dot{\flocopt}_{\Spart}\prl{\vectbf{\stateP}, \boldsymbol{\Sradius}} < 0$ in a neighborhood of $\vectbf{\stateP}^*$, and so it is locally asymptotically stable \cite{hirsch_smale_devaney_DifferentialEquationsDynamicalSystems_2003}.
\end{proof}


\subsection{Congestion Control of Unassigned Robots}
\label{sec.UnassignedRobots}
 
In this subsection we shall present a heuristic congestion management strategy for unassigned robots that improves assigned robots' progress. 

For a choice of vectors of body radii $\boldsymbol{\Bradius} \sqz{\in} \prl{\R_{\geq 0}}^n$, safety margins $\boldsymbol{\safemargin} \sqz{\in} \prl{\R_{>0}}^n$ and sensory footprint radii $\boldsymbol{\Sradius} \sqz{\in} \prl{\R_{\geq0}}^n $, let $\vectbf{\stateP} \in \confspace\prl{\workspace, \boldsymbol{\Bradius}\sqz{+}\boldsymbol{\safemargin}}$ be a multirobot configuration in $\workspace$ with the associated body diagram $\Bpart\prl{\vectbf{\stateP}, \boldsymbol{\Bradius}\sqz{+}\boldsymbol{\safemargin}} = \crl{\Bcell_1, \ldots, \Bcell_n}$, free subdiagram $\Fpart\prl{\vectbf{\stateP}, \boldsymbol{\Bradius} \sqz{+} \boldsymbol{\safemargin}} = \crl{\Fcell_1, \ldots, \Fcell_n}$ and sensory diagram $\Spart\prl{\vectbf{\stateP}, \boldsymbol{\Sradius}} = \crl{\Scell_1, \ldots, \Scell_n}$.

Consider the following heuristic management of robots: 
if $i$th robot has a sensory cell $\Pcell_i$ with a nonempty interior, then it is assigned to the coverage task with sensory cell $\Scell_i$; otherwise, since the robot becomes redundant for the coverage task, it is assigned to move towards a safe location in $\Bcell_i$.   
We therefore define the set of \emph{``active''} domains $\Apart\prl{\vectbf{\stateP}, \boldsymbol{\Bradius}\sqz{+}\boldsymbol{\safemargin}, \boldsymbol{\Sradius}} = \crl{\Acell_1, \Acell_2, \ldots, \Acell_n}$ of robots as
\begin{equation} \label{eq.Acell}
\Acell_i \ldf \left\{
\begin{array}{ll}
\Scell_i & \text{, if } \mathring{\Scell_i} \neq \emptyset, \\
\Bcell_i & \text{, otherwise.}
\end{array}
\right.
\end{equation}
Note that $\Apart\prl{\vectbf{\stateP}, \boldsymbol{\Bradius}\sqz{+}\boldsymbol{\safemargin}, \boldsymbol{\Sradius}}$ defines a cover of $\workspace$ and its elements have nonempty interior  for all $\vectbf{\stateP} \in \confspace\prl{\workspace, \boldsymbol{\Bradius} \sqz{+} \boldsymbol{\safemargin}}$ (\refprop{prop.CollisionPowerDiagram}).

For the first order robot dynamics \refeqn{eq.SystemModel}, we propose the following ``move-to-constrained-active-centroid'' law
\begin{equation} \label{eq.move2ctrdactive}
\vect{\ctrlinput}_i = - \ctrlgain\prl{\vect{\stateP}_i - \overline{\ctrd}_{\Acell_i}},
\end{equation}
that steers each robot towards the constrained centroid, $\overline{\ctrd}_{\Acell_i}$ as defined in \refeqn{eq.ConstrainedCtrd}, of its active domain, $\Acell_i$, which is the closest point in $\overline{\Fcell}_i$ to the centroid $\ctrd_{\Acell_i}$ and so uniquely solves \cite{boyd_vandenberghe_ConvexOptimization_2004}
\begin{equation}\label{eq.ConstrainedOptimizationActive}
\begin{split}
\text{minimize}  &\quad \norm{\vect{\stateQ}_i - \ctrd_{\Acell_i}}^2 \\
\text{subject to} &\quad  \vect{\stateQ}_i \in \overline{\Fcell}_i
\end{split}
\end{equation}
where $\overline{\Fcell}_i$ is convex and $\ctrlgain \in \R_{> 0}$ is a fixed control gain. 
Once again, we assume that $\Apart\prl{\vectbf{\stateP}, \boldsymbol{\Bradius} \sqz{+}\boldsymbol{\safemargin}, \boldsymbol{\Sradius}}$ and $\Fpart\prl{\vectbf{\stateP}, \boldsymbol{\Bradius} \sqz{+}\boldsymbol{\safemargin}}$ are continuously updated.
It is also useful to have  $\goal_{\Apart}\prl{\workspace, \boldsymbol{\Bradius} \sqz{+} \boldsymbol{\safemargin}, \boldsymbol{\Sradius}}$ denote the set of equilibria of the ``move-to-constrained-active-centroid'' law where robots are located at the constrained centroid of their active domains,
\begin{equation}
\!\! \goal_{\Apart}\!\prl{\workspace, \boldsymbol{\Bradius} \sqz{+} \boldsymbol{\safemargin}, \boldsymbol{\Sradius}} \sqz{\ldf} \crl{ \vectbf{\stateP} \sqz{\in} \confspace\prl{\workspace, \boldsymbol{\Bradius} \sqz{+} \boldsymbol{\safemargin}} \! \Big|  \,\vect{\stateP}_i \sqz{=} \overline{\ctrd}_{\Acell_i} ~ \forall i}\!.\!\!\!
\end{equation}

We summarize some important  properties of our ``move-to-constrained-active-centroid'' law as follows:
\begin{proposition}\label{thm.ActiveQualitative}
For any  $\boldsymbol{\Bradius}, \boldsymbol{\Sradius} \sqz{\in} \prl{\R_{\geq 0}}^n$ and  $\boldsymbol{\safemargin} \sqz{\in} \prl{\R_{> 0}}^n$,  
the ``move-to-constrained-active-centroid'' law in \refeqn{eq.move2ctrdactive}
leaves the configuration space of  nonintersecting disks $\confspace\!\prl{\workspace, \boldsymbol{\Bradius} \sqz{+} \boldsymbol{\safemargin}}$  positively invariant; and its unique, continuous and piecewise  differentiable flow, starting at any configuration in $\confspace\!\prl{\workspace, \boldsymbol{\Bradius} \sqz{+} \boldsymbol{\safemargin}}$,  asymptotically reaches $\goal_{\Apart}\!\prl{\workspace, \boldsymbol{\Bradius} \sqz{+} \boldsymbol{\safemargin}, \boldsymbol{\Sradius}}$ without increasing the utility function $\flocopt_{\Spart}\!\prl{\cdot, \boldsymbol{\Sradius}}$ \refeqn{eq.flocopt_heterogeneous} along the way.
\end{proposition}
\begin{proof}
The positive invariance of $\confspace\prl{\workspace, \boldsymbol{\Bradius} \sqz{+} \boldsymbol{\safemargin}}$ under the ``move-to-constrained-active-centroid'' law  and the existence, uniqueness and continuity properties of its flow follow the same pattern as established in \refthm{thm.CollisionQualitative}.

For the stability analysis, using \refeqn{eq.HeterogeneousGradient}, \refeqn{eq.move2ctrdactive} and \refeqn{eq.ConstrainedOptimizationActive}, one can show that the continuously differentiable  utility function $\flocopt_{\Spart}\prl{., \boldsymbol{\Sradius}}$ \refeqn{eq.flocopt_heterogeneous} is nonincreasing along the trajectory of the ``move-to-constrained-active-centroid'' law as follows:
\begin{align}
\dot{\flocopt}_{\Spart}\prl{\vectbf{\stateP}, \boldsymbol{\Sradius}} &= - k \!\!\!\sum_{\substack{i \in \crl{1, \ldots,n}\\ \mathring{\Scell_i} \neq \emptyset}}\!\!\mass_{\Scell_i}\hspace{-9mm}\underbrace{2\vectprod{\prl{\vect{\stateP}_i \sqz{-} \ctrd_{\Scell_i}}}{\!\prl{\vect{\stateP}_i \sqz{-} \overline{\ctrd}_{\Scell_i}}}}_{\substack{\geq \norm{\vect{\stateP}_i - \overline{\ctrd}_{\Scell_i}}^2, \\ \text{ since } \vect{\stateP}_i \in \Fcell_i \text{ and } \norm{\vect{\stateP}_i - \ctrd_{\Scell_i}}^2 \geq \norm{\overline{\ctrd}_{\Scell_i} - \ctrd_{\Scell_i}}^2} }  \nonumber\\
& \hspace{0mm}-k \!\!\! \sum_{\substack{i \in \crl{1, \ldots,n}\\ \mathring{\Scell_i} = \emptyset}} \!\!\!\underbrace{~\mass_{\Scell_i}~}_{\substack{= 0 \\ \text{ since } \mathring{\Scell}_i =  \emptyset}}\! \!\!\! 2\vectprod{\prl{\vect{\stateP}_i \sqz{-} \ctrd_{\Scell_i}}}{\!\prl{\vect{\stateP}_i \sqz{-} \overline{\ctrd}_{\Bcell_i}}}\!, \!\! \\ 
& \leq - k \sum_{\substack{i \in \crl{1,\ldots,n} \\ \mathring{\Scell}_i \neq \emptyset}}\mass_{\Scell_i} \norm{\vect{\stateP}_i - \overline{\ctrd}_{\Scell_i}}^2 \leq 0. 
\end{align}
Hence, we have from Lasalle's Invariance Principle \cite{khalil_NonlinearSystems_2001} that,  at an equilibrium point of the ``move-to-constrained-active-centroid'' law, a robot is located at the constrained centroid, $\overline{\ctrd}_{\Scell_i}$, of its sensory cell, $\Scell_i$, if it has a nonempty interior, i.e. $\mathring{\Scell}_i \neq \emptyset$.
Given that $\vect{\stateP}_i = \overline{\ctrd}_{\Scell_i}$ for all $i \in \crl{1, \ldots,n}$ with $\mathring{\Scell}_i \neq \emptyset$, using \refeqn{eq.HeterogeneousGradient}, \refeqn{eq.move2ctrdactive} and  \refeqn{eq.ConstrainedOptimizationActive}, one can further show that 
\begin{align}
\dot{\flocopt}_{\Bpart}\prl{\vectbf{\stateP}, \boldsymbol{\Bradius} \sqz{+} \boldsymbol{\safemargin}} &= - k \!\!\!\sum_{\substack{i \in \crl{1, \ldots,n}\\ \mathring{\Scell_i} \neq \emptyset}}\!\!\mass_{\Bcell_i}\hspace{0mm}\underbrace{2\vectprod{\prl{\vect{\stateP}_i \sqz{-} \ctrd_{\Bcell_i}}}{\!\prl{\vect{\stateP}_i \sqz{-} \overline{\ctrd}_{\Scell_i}}}}_{\substack{= 0, \\ \text{ since } \vect{\stateP}_i = \overline{\ctrd}_{\Scell_i}  }}  \nonumber\\
& \hspace{0mm}-k \!\!\! \sum_{\substack{i \in \crl{1, \ldots,n}\\ \mathring{\Scell_i} = \emptyset}} \!\!\!\!\!\mass_{\Bcell_i} \hspace{-9mm} \underbrace{2\vectprod{\prl{\vect{\stateP}_i \sqz{-} \ctrd_{\Bcell_i}}}{\!\prl{\vect{\stateP}_i \sqz{-} \overline{\ctrd}_{\Bcell_i}}}}_{\substack{\geq \norm{\vect{\stateP}_i - \overline{\ctrd}_{\Bcell_i}}^2, \\ \text{ since } \vect{\stateP}_i \in \Fcell_i \text{ and } \norm{\vect{\stateP}_i - \ctrd_{\Bcell_i}}^2 \geq \norm{\overline{\ctrd}_{\Bcell_i} - \ctrd_{\Bcell_i}}^2}} \hspace{-9mm}, \!\! \\ 
& \leq - k \sum_{\substack{i \in \crl{1,\ldots,n} \\ \mathring{\Scell}_i = \emptyset}}\mass_{\Bcell_i} \norm{\vect{\stateP}_i - \overline{\ctrd}_{\Bcell_i}}^2 \leq 0. 
\end{align}
Therefore, at a stationary point of \refeqn{eq.move2ctrdactive} $i$th robot is located at the constrained centroid, $\overline{\ctrd}_{\Bcell_i}$, of its body cell $\Bcell_i$ if  $\mathring{\Scell}_i = \emptyset$.
Overall, by Lasalle's Invariance Principle, we have that any multirobot configuration starting in $\confspace\prl{\workspace, \boldsymbol{\Bradius} \sqz{+}\boldsymbol{\safemargin}}$ asymptotically converges to a locally optimal sensing configuration in $\goal_{\Apart}\prl{\workspace, \boldsymbol{\Bradius} \sqz{+} \boldsymbol{\safemargin}, \boldsymbol{\Sradius}}$, which completes the proof.  
\end{proof}

\subsection{Coverage Control of Differential Drive Robots}
\label{sec.DifferentialDriveRobots}


Consider a noncolliding placement of a heterogeneous group of disk-shaped differential drive robots  $\prl{\vectbf{\stateP}, \boldsymbol{\theta}} \in \confspace\prl{\workspace, \boldsymbol{\Bradius} \sqz{+} \boldsymbol{\safemargin}} \times (-\pi, \pi]^n$  in a convex planar environment $\workspace \subset \R^2$  with associated vectors of body radii $\boldsymbol{\Bradius} \in \prl{\R_{\geq0}}^n$, safety margins $\boldsymbol{\safemargin} \in \prl{\R_{>0}}^n$ and sensory footprint radii $\boldsymbol{\Sradius} \in \prl{\R_{\geq0}}^n$, where $\boldsymbol{\theta} = \prl{\theta_1, \theta_2, \ldots, \theta_n}$ is the vector of robot orientations.

The kinematic equations describing the motion of each differential drive robot are
\begin{equation}
\begin{split}
\dot{\vect{\stateP}}_i &= \ctrlvel_i \twovec{\cos \theta_i}{\sin \theta_i},\\
\dot{\theta}_i &= \ctrlang_i,
\end{split}
\end{equation}   
where $\ctrlvel_i \in \R$ and $\ctrlang_i \in \R$ are, respectively, the linear (tangential) and angular velocity inputs of $i$th robot.
Note that the differential drive model is underactuated due to the nonholonomic constraint $\scalebox{0.8}{$\tr{\twovec{-\sin \theta_i}{\cos \theta_i}}$}\dot{\vect{\stateP}}_i = 0$.

Let $\Spart\prl{\vectbf{\stateP}, \boldsymbol{\Sradius}} = \crl{\Scell_1, \ldots, \Scell_n}$ \refeqn{eq.Pcell} be the sensory diagram of $\workspace$ based on robot locations $\vectbf{\stateP}$ and sensory footprint radii $\boldsymbol{\Sradius}$, and $\Fpart\prl{\vectbf{\stateP}, \boldsymbol{\Bradius} \sqz{+} \boldsymbol{\safemargin}} = \crl{\Fcell_1, \ldots, \Fcell_n}$ \refeqn{eq.Fcell} be the free subdiagram associated with  configuration  $\vectbf{\stateP}$ and enlarged body radii $\boldsymbol{\Bradius} +\boldsymbol{\safemargin}$. 
For a choice of $\boldsymbol{\safermargin} \in \prl{\R_{>0}}^n$ with $\safermargin_i > \safemargin_i$ for all $i$, we further define $\Tpart\prl{\vectbf{\stateP}, \boldsymbol{\Bradius}\sqz{+}\boldsymbol{\safermargin}} = \crl{\Tcell_1, \Tcell_2, \ldots , \Tcell_n}$ to be
\begin{equation}
\Tcell_i \ldf \conv\prl{\crl{\vect{\stateP}_i} \cup \Fcell_i'}
\end{equation}
where $\Fpart\prl{\vectbf{\stateP}, \boldsymbol{\Bradius} \sqz{+} \boldsymbol{\safermargin}} = \crl{\Fcell_1', \Fcell_2', \ldots, \Fcell_n'}$ and $\conv\prl{A}$ denotes the convex hull of set $A$. 
Note that, since $\Fcell_i' \subset \Fcell_i$, $\vect{\stateP}_i \in \Fcell_i$ and $\Fcell_i$ is convex, every element of $\Tpart\prl{\vectbf{\stateP}, \boldsymbol{\Bradius}\sqz{+}\boldsymbol{\safermargin}}$ is contained in the associated element of $\Fpart\prl{\vectbf{\stateP}, \boldsymbol{\Bradius} \sqz{+} \boldsymbol{\safemargin}}$, i.e. $\Tcell_i \subseteq \Fcell_i$. 
It is useful to remark that we particularly require $\vect{\stateP}_i \in \Tcell_i$ to guarantee an optimal choice of a local target position in \refeqn{eq.ConstrainedOptimizationAngular} relative to $\vect{\stateP}_i$, and we construct  subset $\Tcell_i$ of $\Fcell_i$ to increase the convergence rate of our proposed coverage control law in \refeqn{eq.move2ctrdcollisiondiff}.

As in the case of ``move-to-constrained-centroid'' law of fully actuated robots in \refeqn{eq.move2ctrdcollision}, for optimal coverage  each differential drive robot will intent to move towards the constrained centroid, $\overline{\ctrd}_{\Scell_i}$ \refeqn{eq.ConstrainedCtrd}, of its sensory cell, $\Scell_i$, but with a slight difference due to the nonholonomic constraint. 
To determine a linear velocity input guaranteeing collision avoidance, we select a safe target location that solves the following convex programming,   
\begin{equation}\label{eq.ConstrainedOptimizationDiff}
\begin{split}
\text{minimize}  &\quad \norm{\vect{\stateQ}_i - \ctrd_{\Scell_i}}^2 \\
\text{subject to} &\quad  \vect{\stateQ}_i \in \overline{\Fcell}_i \cap H_i
\end{split}
\end{equation}
where 
\begin{equation}
H_i \ldf \crl{\vect{x} \in \workspace \Big| ~  \scalebox{0.9}{$\tr{\twovec{\cos \theta_i}{\sin \theta_i}}$}\prl{\vect{x} - \vect{\stateP}_i} = 0}
\end{equation} 
is the straight line motion range due to the nonholonomic constraint.
Note that $\overline{\Fcell}_i \cap H_i$ is a closed line segment in $\workspace$.
Hence, once again, the unique solution of  \refeqn{eq.ConstrainedOptimizationDiff} is given by
\begin{equation}\label{eq.ConstrainedCtrdDiff}
\overline{\ctrd}_{\Scell_i}^{\,\ctrlvel} \ldf \left \{ 
\begin{array}{l@{}l}
\ctrd_{\Scell_i} & \text{, if } \ctrd_{\Scell_i} \in  \overline{\Fcell}_i \cap H_i, \\
\Pi_{\overline{\Fcell}_i \cap H_i}\prl{\ctrd_{\Scell_i}} & \text{, otherwise},
\end{array}
\right.
\end{equation}
where $\Pi_{C}$ is the metric projection map onto a convex set $C$.
Similarly, to determine robot's angular motion, we select another safe target location that solves 
\begin{equation}\label{eq.ConstrainedOptimizationAngular}
\begin{split}
\text{minimize}  &\quad \norm{\vect{\stateQ}_i - \ctrd_{\Scell_i}}^2 \\
\text{subject to} &\quad  \vect{\stateQ}_i \in \overline{\Tcell}_i 
\end{split}
\end{equation} 
where $\overline{\Tcell}_i \subset \overline{\Fcell}_i$ is convex, and the unique solution of \refeqn{eq.ConstrainedOptimizationAngular} is
\begin{equation}\label{eq.ConstrainedCtrdAngular}
\overline{\ctrd}_{\Scell_i}^{\,\ctrlang} \ldf \left \{ 
\begin{array}{l@{}l}
\ctrd_{\Scell_i} & \text{, if } \ctrd_{\Scell_i} \in  \overline{\Tcell}_i, \\
\Pi_{\overline{\Tcell}_i}\prl{\ctrd_{\Scell_i}} & \text{, otherwise}.
\end{array}
\right.
\end{equation}

\begin{figure*}[t]
\centering
\begin{tabular}{@{}c@{\hspace{0.5mm}}c@{\hspace{0.5mm}}c@{\hspace{0.5mm}}c@{\hspace{0.5mm}}c@{}}
\includegraphics[width=0.2\textwidth]{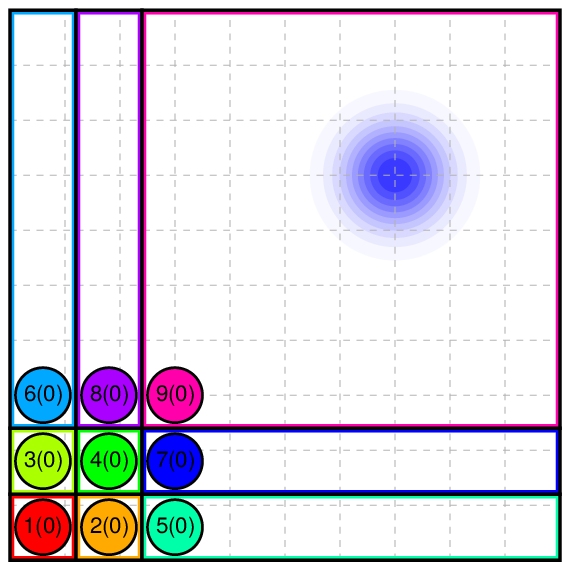} 
&
\includegraphics[width=0.2\textwidth]{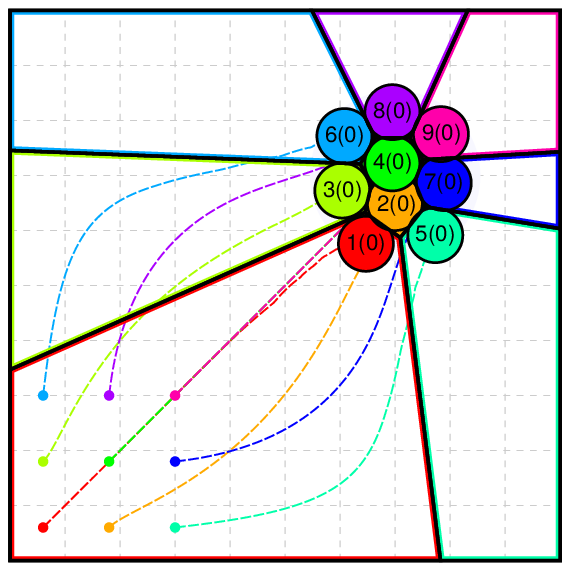} 
&
\includegraphics[width=0.2\textwidth]{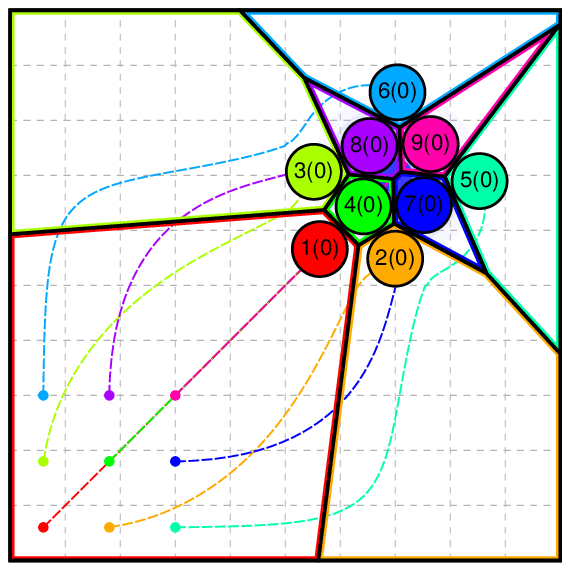} 
&
\includegraphics[width=0.2\textwidth]{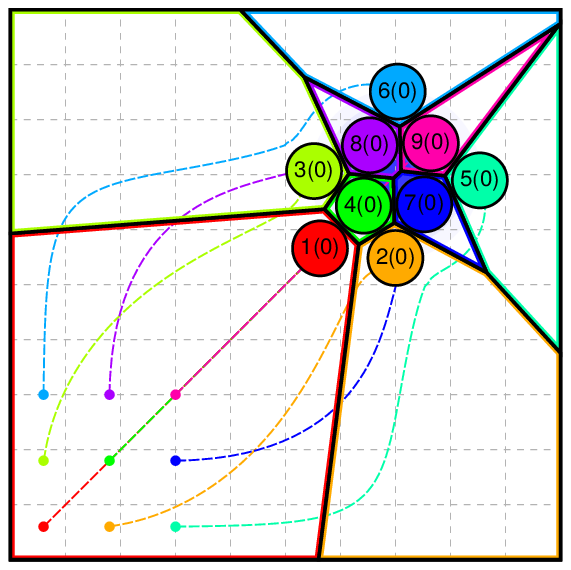} 
&
\includegraphics[width=0.2\textwidth]{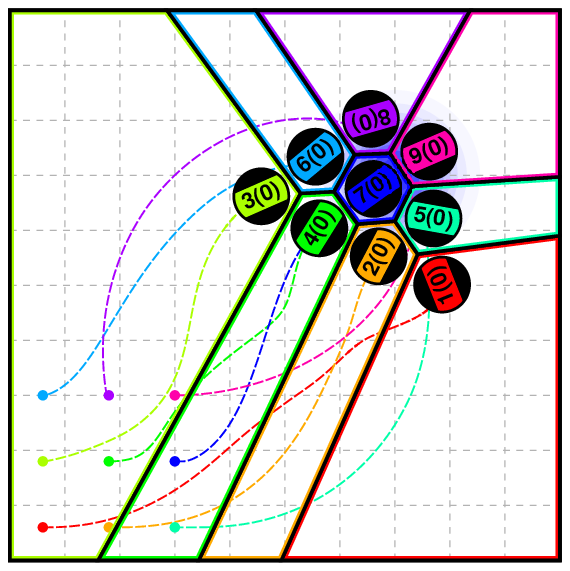} \\[-1.5mm]
\scalebox{0.8}{(a)} & \scalebox{0.8}{(b)} & \scalebox{0.8}{(c)} & \scalebox{0.8}{(d)} & \scalebox{0.8}{(e)}
\end{tabular}
\vspace{-4mm}
\caption{Avoiding collisions around a concentrated event distribution. (a) Initial configuration of a homogeneous robot network, where the weight of sensory cell are shown in the parenthesis, and the resulting  trajectories of (b) the standard ``move-centroid'' law \refeqn{eq.move2ctrdheterogeneous}, (c) the ``move-to-constrained-centroid'' law \refeqn{eq.move2ctrdcollision}, (d) the ``move-to-constrained-active-centroid'' law \refeqn{eq.move2ctrdactive}, (e) the ``move-to-constrained-centroid'' law of differential drive robots \refeqn{eq.move2ctrdcollisiondiff} which are initially aligned with the horizontal axis.}
\label{fig.ConcentratedDistribution}
\vspace{-3mm}
\end{figure*}

\begin{figure*}
\centering
\begin{tabular}{@{}c@{\hspace{0.5mm}}c@{\hspace{0.5mm}}c@{\hspace{0.5mm}}c@{\hspace{0.5mm}}c@{\hspace{0.5mm}}@{}}
\includegraphics[width=0.2\textwidth]{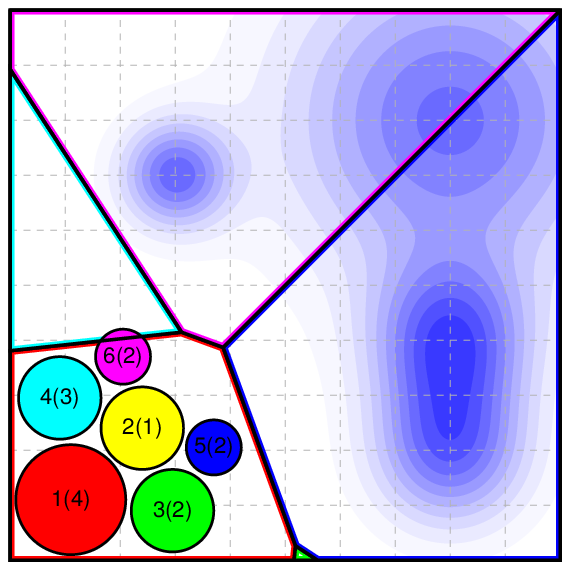} 
&
\includegraphics[width=0.2\textwidth]{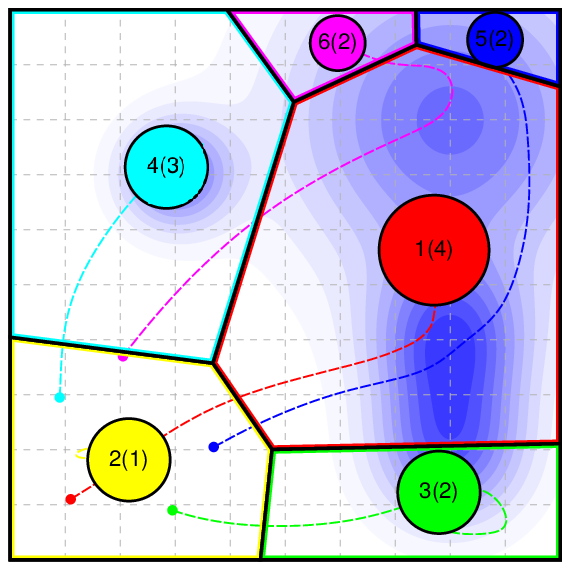} 
&
\includegraphics[width=0.2\textwidth]{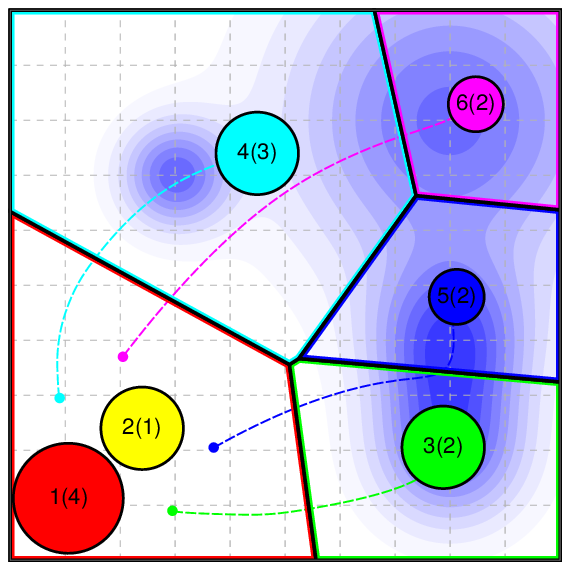} 
&
\includegraphics[width=0.2\textwidth]{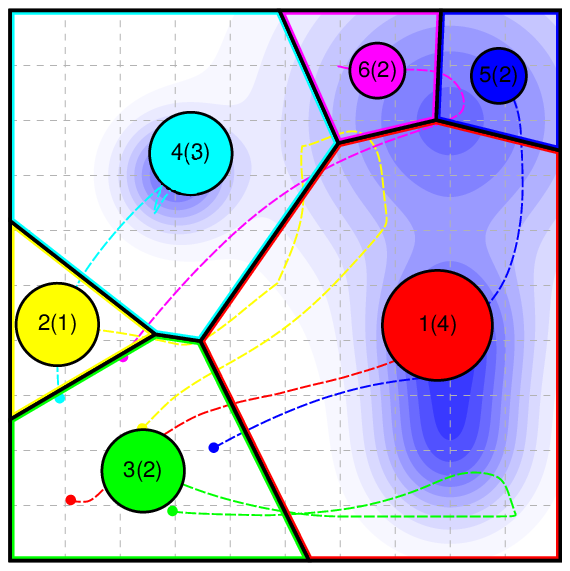} 
&
\includegraphics[width=0.2\textwidth]{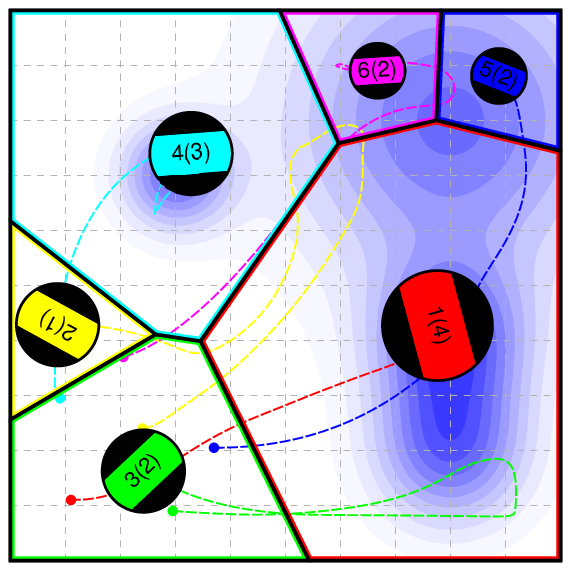} \\[-1.5mm]
\scalebox{0.8}{(a)} & \scalebox{0.8}{(b)} & \scalebox{0.8}{(c)} & \scalebox{0.8}{(d)} & \scalebox{0.8}{(e)}
\end{tabular}
\vspace{-4mm}
\caption{Safe coverage control of heterogeneous disk-shaped robots with a heuristic management of unassigned robots. (a) Initial configuration of a heterogeneous robot network, where the weight of sensory cell are shown in the parenthesis, and the resulting  trajectories of (b) the standard ``move-centroid'' law \refeqn{eq.move2ctrdheterogeneous}, (c) the ``move-to-constrained-centroid'' law \refeqn{eq.move2ctrdcollision}, (d) the ``move-to-constrained-active-centroid'' law \refeqn{eq.move2ctrdactive}, (e) the ``move-to-constrained-active-centroid'' law of differential drive robots which are initially aligned with the horizontal axis. }
\label{fig.UnassignedRobots}
\vspace{-2mm}
\end{figure*}

Accordingly, based on a standard differential drive controller \cite{astolfi_JDSMC1999}, we propose the following ``move-to-constrained-centroid'' law for differential drive robots,\footnotemark
\begin{subequations} \label{eq.move2ctrdcollisiondiff} 
\begin{align}
\ctrlvel_i &= - \ctrlgain ~\scalebox{0.8}{$\tr{\twovec{\cos \theta_i}{\sin \theta_i}}$}\prl{\vect{\stateP}_i  - \overline{\ctrd}_{\Scell_i}^{\, \ctrlvel}}, \\
\ctrlang_i &= \ctrlgain~ \atan\prl{\frac{\scalebox{0.8}{$\tr{\twovec{-\sin \theta_i}{\cos \theta_i}}$}\prl{\vect{\stateP}_i  - \overline{\ctrd}_{\Scell_i}^{\,\ctrlang}}}{\scalebox{0.8}{$\tr{\twovec{\cos \theta_i}{\sin \theta_i}}$}\prl{\vect{\stateP}_i  - \overline{\ctrd}_{\Scell_i}^{\,\ctrlang}}}}, 
\end{align} 
\end{subequations}
where $\ctrlgain > 0$ is fixed. \footnotetext{To resolve indeterminacy we set $\ctrlang_i = 0$ whenever $\vect{\stateP}_i  = \overline{\ctrd}^{\,\ctrlang}_{\Scell_i}$.}
%
%
Having $\goal_{\mathcal{D}}\!\prl{\workspace, \boldsymbol{\Bradius}, \boldsymbol{\safemargin}, \boldsymbol{\safermargin}, \boldsymbol{\Sradius}}$ denote  its set of stationary points where the constrained centroids $\overline{\ctrd}_{\Scell_i}^{\,\ctrlvel}$ and $\overline{\ctrd}_{\Scell_i}^{\,\ctrlang}$ coincide and  $i$th robot is located at $\overline{\ctrd}_{\Scell_i}^{\,\ctrlvel} = \overline{\ctrd}_{\Scell_i}^{\ctrlang}$,
\begin{equation}
 \goal_{\mathcal{D}}\!\prl{\workspace, \boldsymbol{\Bradius}, \boldsymbol{\safemargin}, \boldsymbol{\safermargin}, \boldsymbol{\Sradius}} \sqz{\ldf} \crl{\!\vectbf{\stateP} \sqz{\in} \confspace\!\prl{\workspace, \boldsymbol{\Bradius} \sqz{+} \boldsymbol{\safemargin}} \! \Big|  \vect{\stateP}_i \sqz{=} \overline{\ctrd}_{\Scell_i}^{\,\ctrlvel} \sqz{=} \overline{\ctrd}_{\Scell_i}^{\,\ctrlang}~  \forall i }\!, \nonumber
\end{equation}
we summarize important qualitative properties of the ``move-to-constrained-centroid'' law of  differential drive robots as:
\begin{proposition}
For any  $\boldsymbol{\Bradius}, \boldsymbol{\Sradius} \in \prl{\R_{\geq 0}}^n$ and $\boldsymbol{\safemargin}, \boldsymbol{\safermargin} \in \prl{\R_{>0}}^n$ with $\safemargin_i < \safermargin_i$ for all $i$,
the ``move-to-constrained-centroid'' law of differential drive robots in \refeqn{eq.move2ctrdcollisiondiff} asymptotically steers all configurations in its positively invariant domain  $\confspace\prl{\workspace, \boldsymbol{\Bradius}\sqz{+}\boldsymbol{\safemargin}} \times (-\pi, \pi]^n$ towards the set of optimal sensing configurations  $\goal_{\mathcal{D}}\prl{\workspace, \boldsymbol{\Bradius}, \boldsymbol{\safemargin}, \boldsymbol{\safermargin}, \boldsymbol{\Sradius}} \times (-\pi,\pi]^n$ without increasing the utility function $\flocopt_{\Spart}\prl{\cdot, \boldsymbol{\Sradius}}$ \refeqn{eq.flocopt_heterogeneous} along the way.   
\end{proposition}
\begin{proof}
The configuration space $\confspace\prl{\workspace, \boldsymbol{\Bradius}\sqz{+}\boldsymbol{\safemargin}} \sqz{\times} (-\pi, \pi]^n$ is positively invariant under the ``move-to-constrained-centroid'' law in \refeqn{eq.move2ctrdcollisiondiff} because, by construction \refeqn{eq.ConstrainedOptimizationDiff}, each robot's motion is constrained to the associated safe partition subcell in $\workspace$.
The existence and uniqueness of its flow can be established using the pattern of the proof of \refthm{thm.CollisionQualitative} and the flow properties of the differential drive controller in \cite{astolfi_JDSMC1999}.

Now, using $\flocopt_{\Spart}\prl{\cdot, \boldsymbol{\Sradius}}$ \refeqn{eq.flocopt_heterogeneous} as a continuously differentiable Lyapunov function, we obtain the stability properties as follows: for any $\prl{\vectbf{\stateP}, \boldsymbol{\theta}} \in \confspace\prl{\workspace, \boldsymbol{\Bradius}\sqz{+} \boldsymbol{\safemargin}} \times (-\pi, \pi]^n$
\begin{align}
\dot{\flocopt}_{\Spart}\prl{\vectbf{\stateP},\boldsymbol{\Sradius}} &= -k \sum_{i =1}^{n} \mass_{\Scell_i} \hspace{-10mm}\underbrace{2\vectprod{\prl{\vect{\stateP}_i - \ctrd_{\Scell_i}}}{\prl{\vect{\stateP}_i - \overline{\ctrd}_{\Scell_i}^{\, \ctrlvel}}}}_{\substack{\geq \norm{\vect{\stateP}_i - \overline{\ctrd}_{\Scell_i}^{\,\ctrlvel}}^2, \\ \text{ since }  \vect{\stateP}_i \in \Fcell_i \cap H_i \text{ and }\norm{\vect{\stateP}_i - \ctrd_{\Scell_i}}^2 \geq \norm{\overline{\ctrd}_{\Scell_i}^{\, \ctrlvel} - \ctrd_{\Scell_i}}^2}} \hspace{-10mm}, \\
& \leq -k \sum_{i =1}^{n} \mass_{\Scell_i} \norm{\vect{\stateP}_i - \overline{\ctrd}_{\Scell_i}^{\,\ctrlvel}}^2 \leq 0,
\end{align} 
where $\dot{\vect{\stateP}}_i = -\ctrlgain\prl{\vect{\stateP}_i - \overline{\ctrd}_{\Scell_i}^{\,\ctrlvel}}$.
Hence, by LaSalle's Invariance Principle \cite{khalil_NonlinearSystems_2001}, at a stationary point of \refeqn{eq.move2ctrdcollisiondiff} $i$th robot is located at $\overline{\ctrd}_{\Scell_i}^{\,\ctrlvel}$. 
Since for fixed  $\overline{\ctrd}_{\Scell_i}^{\, \ctrlvel}$ and $\overline{\ctrd}_{\Scell_i}^{\, \ctrlang}$  the standard differential drive controller  asymptotically aligns each robot with the constrained centroid $\overline{\ctrd}_{\Scell_i}^{\, \ctrlang}$, i.e. $\scalebox{0.8}{$\tr{\twovec{-\sin \theta_i}{\cos \theta_i}}$}\!\!\prl{\vect{\stateP}_i  \sqz{-} \overline{\ctrd}_{\Scell_i}^{\, \ctrlang}} = 0$  \cite{astolfi_JDSMC1999}, it is guaranteed by \refeqn{eq.ConstrainedOptimizationDiff} and \refeqn{eq.ConstrainedOptimizationAngular} that  $\overline{\ctrd}_{\Scell_i}^{\,\ctrlvel} = \overline{\ctrd}_{\Scell_i}^{\,\ctrlang}$ whenever $\norm{\vect{\stateP}_i - \overline{\ctrd}_{\Scell_i}^{\,\ctrlvel}}= 0$ and $\scalebox{0.8}{$\tr{\twovec{-\sin \theta_i}{\cos \theta_i}}$}\!\!\prl{\vect{\stateP}_i  \sqz{-} \overline{\ctrd}_{\Scell_i}^{\,\ctrlang}} = 0$.  
Therefore, we have from LaSalle's Invariance Principle that all configurations in $\confspace\prl{\workspace, \boldsymbol{\Bradius} \sqz{+} \boldsymbol{\safemargin}} \times (-\pi,\pi]^n$ asymptotically reach $\goal_{\mathcal{D}}\prl{\workspace, \boldsymbol{\Bradius}, \boldsymbol{\safemargin}, \boldsymbol{\safermargin}, \boldsymbol{\Sradius}} \times (-\pi,\pi]^n$. 
\end{proof}

Finally, note that the ``move-to-constrained-active-centroid'' law of  \refsec{sec.UnassignedRobots} can be utilized for congestion control of differential drive robots by using active domains in \refeqn{eq.Acell} instead of the sensory diagram $\Spart\prl{\vectbf{\stateP}, \boldsymbol{\Sradius}}$, and the resulting coverage law maintains qualitative properties.

\section{Numerical Simulations}
\label{sec.NumericalSimulation}

A common source of collisions between robots while performing a distributed sensing task is a concentrated event distribution which generally causes robots to move towards the same small region of the environment.\footnotemark ~  
We therefore consider the following event distribution, $\fevent:\brl{0,10}^2\rightarrow \R_{>0}$, for a homogeneous group of disk-shaped robots operating in a $10 \times 10$ square environment, 
\footnotetext{For all simulations we set $\safemargin_i = 0.05$  and $\safermargin_i = 0.1$ for all $i \in \crl{1,2, \ldots, n}$, and all simulations
are obtained through numerical integration of the associated coverage control law using the \texttt{ode45} function of MATLAB, and the computation of the centroid of a power cell in \refeqn{eq.mass_ctrd} is approximated  by discretizing the power cell by a $20\times20$ grid.}
\begin{equation}
\fevent\prl{\vect{\stateQ}} = e^{-\norm{\vect{\stateQ} - \scalebox{0.6}{$\twovec{7}{7}$}}^2} .
\end{equation} 
In \reffig{fig.ConcentratedDistribution} we present  the resulting trajectories of our proposed coverage control algorithms. 
Since the event distribution is concentrated around a small region, as expected, the standard ``move-to-centroid'' law  steers robots to a centroidal Voronoi configuration where robots collide.
On the other hand, since a Voronoi partition  has no occupancy defect,  our ``move-to-constrained-centroid'' and ``move-to-constrained-active-centroid'' laws yield the same trajectory  that asymptotically converges a collision free constrained centroidal Voronoi configuration.
It is also well known that minimizing the location optimization function $\flocopt_{\Spart}$ \refeqn{eq.flocopt_heterogeneous} generally results in a locally optimal sensing configuration, and  we observe in Figures \ref{fig.ConcentratedDistribution}.(c) and \ref{fig.ConcentratedDistribution}.(e) that,  although they are initiated at the same location,  fully actuated and differential drive robots asymptotically reach different constrained centroidal Voronoi configurations.     

To demonstrate how unassigned robots may limit the mobility of others, we consider a heterogeneous group of disk-shaped robots operating in a $10 \times 10$ environment with the following event distribution function, $\fevent:\brl{0,10}^2 \rightarrow \R_{>0}$, 
\begin{equation}
\begin{split}
\fevent\prl{\vect{\stateQ}} = 1 &+ 10 e^{-\frac{1}{9}\norm{\vect{\stateQ} - \scalebox{0.6}{$\twovec{8}{8}$}}^2} + e^{-\frac{1}{2}\norm{\vect{\stateQ} - \scalebox{0.6}{$\twovec{8}{2}$}}^2} \\
& \hspace{15mm}+ e^{-\frac{1}{2}\norm{\vect{\stateQ} - \scalebox{0.6}{$\twovec{8}{4}$}}^2} + e^{-\norm{\vect{\stateQ} - \scalebox{0.6}{$\twovec{3}{7}$}}^2},
\end{split}
\end{equation}
which is also used in \cite{kwok_martinez_IJRNC2010}.
In \reffig{fig.UnassignedRobots} we illustrate  the resulting trajectories of our safe coverage control algorithms. 
As seen in \reffig{fig.UnassignedRobots}.(a), the 2nd robot is initially not assigned to any region.
It stays stationary for a certain finite time under the the standard ``move-to-centroid'' law during which the 1st robot moves through it. 
Also notice that the 3rd robot violates the workspace boundary before converging a safe location. 
In summary, the ``move-to-centroid'' law steers disk-shaped robots  to a locally optimal sensing configuration without avoiding collisions along the way.  
Our ``move-to-constrained-centroid'' law prevents any possible self-collisions and collisions with the boundary of the environment.
However, since the 2nd robot stays unassigned for all future time, the 1st robot is blocked and it can not move to a better coverage location.
Fortunately, while guaranteeing collision avoidance, our ``move-to-constrained-active-centroid'' law  steers unassigned robots to improve assigned robots' progress for both fully actuated and differential drive robots,  as illustrated in Figures \ref{fig.UnassignedRobots}.(d) and \ref{fig.UnassignedRobots}.(e), respectively.

\section{Conclusion}
\label{sec.Conclusion}

In this paper we introduce a novel use of power diagrams for identifying collision free multirobot configurations, and  propose a constrained optimization framework combining  coverage control and collision avoidance for fully actuated disk-shaped robots, comprising the main contributions of the paper.
We also present its extensions for the widely used differential drive model and for  congestion management of unassigned robots.
Numerical simulations demonstrate the effectiveness of the proposed coverage control algorithms.

Work now in progress targets another extension of Voronoi-based coverage control for hierarchical settings, based on nested partitions of convex environments \cite{arslan_guralnik_kod_WAFR2014}.
We also believe that encoding collision free configurations in terms of power diagrams might have a significant value for robot motion planning, and  we are currently exploring its possible usage in the design of feedback motion planners.

%

\section*{ACKNOWLEDGMENT}

This work was supported by AFOSR under the CHASE MURI FA9550-10-1-0567.

%


\bibliographystyle{IEEEtran}
\bibliography{IEEEabrv,coveragecontrol}

\end{document}